\documentclass{article}

\usepackage{microtype}
\usepackage{graphicx}
\usepackage{subcaption}
\usepackage{booktabs} 
\usepackage{multirow}
\usepackage{hyperref}

\usepackage[preprint]{icml2026}

\usepackage{amsmath}
\usepackage{amssymb}
\usepackage{mathtools}
\usepackage{amsthm}

\usepackage[capitalize,noabbrev]{cleveref}

\theoremstyle{plain}
\newtheorem{theorem}{Theorem}[section]
\newtheorem{proposition}[theorem]{Proposition}

\newtheorem{corollary}[theorem]{Corollary}
\theoremstyle{definition}
\newtheorem{definition}[theorem]{Definition}

\theoremstyle{remark}

\usepackage[textsize=tiny]{todonotes}

\definecolor{winter}{rgb}{0.85,0.08,0.2}
\definecolor{summer}{rgb}{0.95,0.53,0.18}         
\definecolor{spring}{rgb}{0.02,0.93,0.68}
\definecolor{autumn}{rgb}{0.02,0.68,0.9}
\definecolor{grey}{rgb}{0.4,0.4,0.4}

\definecolor{cred}{HTML}{ffb3b3}

\newcommand\red[1]{{\color{red}#1}}

\icmltitlerunning{Diversity in Large Language Models under Supervised Fine-Tuning}

\begin{document}

\twocolumn[
  \icmltitle{Diversity in Large Language Models under Supervised Fine-Tuning}



  \icmlsetsymbol{equal}{*}

    \begin{icmlauthorlist}
    \icmlauthor{Roman Klypa}{equal,gre}
    \icmlauthor{Oleksandr Cherednichenko}{equal,swe}
    \end{icmlauthorlist}

    \icmlaffiliation{gre}{Univ. Grenoble Alpes, CNRS, Grenoble INP, LJK, 38000 Grenoble, France}
    \icmlaffiliation{swe}{Department of Mathematics and Mathematical Statistics, Integrated Science Lab, Umeå University, Sweden}
    
    \icmlcorrespondingauthor{Roman Klypa}{roman.klypa@univ-grenoble-alpes.fr}
    \icmlcorrespondingauthor{Oleksandr Cherednichenko}{oleksandr.cherednichenko@umu.se}

  \icmlkeywords{Machine Learning, ICML}

  \vskip 0.3in
]



\printAffiliationsAndNotice{\icmlEqualContribution}  

\begin{abstract}
Supervised Fine-Tuning (SFT) is essential for aligning Large Language Models (LLMs) with user intent, yet it is believed to suppress generative diversity. Although this reduction is frequently referenced, formal empirical testing of the phenomenon remains limited. The expressiveness of LLMs by itself was addressed by multiple prior methods. Their varying perspectives suggest that deeper investigation could yield further improvements. In this study, we attribute the decline to two primary drivers: the neglect of low-frequency patterns within fine-tuning datasets and the forgetting of preexisting knowledge. Motivated by our theoretical analysis, we develop Tempered Focal (TOFU) loss, a novel objective that addresses both stated challenges simultaneously. Our extensive evaluation confirms at scale that generation breadth narrows after SFT and strengthens the hypothesis explaining this effect. Across multiple models and benchmarks, we demonstrate that TOFU enhances output diversity while preserving high response quality, offering a principled approach to SFT.
\end{abstract}

\section{Introduction}

Autoregressive language models (LMs) \cite{bengio_neural_2000} have demonstrated remarkable progress in modeling natural language. Increased data availability and model capacity have allowed Transformer-based \cite{vaswani_attention_2023} architectures to generate  text that closely resembles human-written content. Modern Large Language Models (LLMs) now serve as powerful generative engines that excel across a broad range of specialized tasks, from document summarization~\cite{brown_language_2020} to complex reasoning~\cite{wei_chain--thought_2023}.

Despite their impressive capabilities, pretrained LLMs often produce responses that do not fully capture user intents. This limitation arises because models may generate tokens that are statistically plausible yet semantically misaligned with the actual query, resulting in irrelevant or unhelpful output. To address this issue, researchers increasingly rely on instruction tuning \cite{wei_finetuned_2022,chung_scaling_2022,raffel_exploring_2023}, also referred to as Supervised Fine-Tuning (SFT) \cite{ouyang_training_2022, bai_training_2022}. This process refines LLMs on curated collections of high-quality prompt–response pairs, enabling better alignment with user expectations and task-specific objectives. This stage plays a crucial role in preparing models for subsequent reinforcement learning (RL) alignment \cite{ziegler_fine-tuning_2020,ouyang_training_2022,rafailov_direct_2023, shao_deepseekmath_2024}. By first grounding the model in high-quality demonstrations, SFT provides a stable initialization that enables RL methods to effectively refine task-specific behaviors, rather than struggling with unstructured or misaligned model outputs.

\begin{figure*}[ht]
\vskip 0.2in
\begin{center}
\centerline{\includegraphics[width=\linewidth]{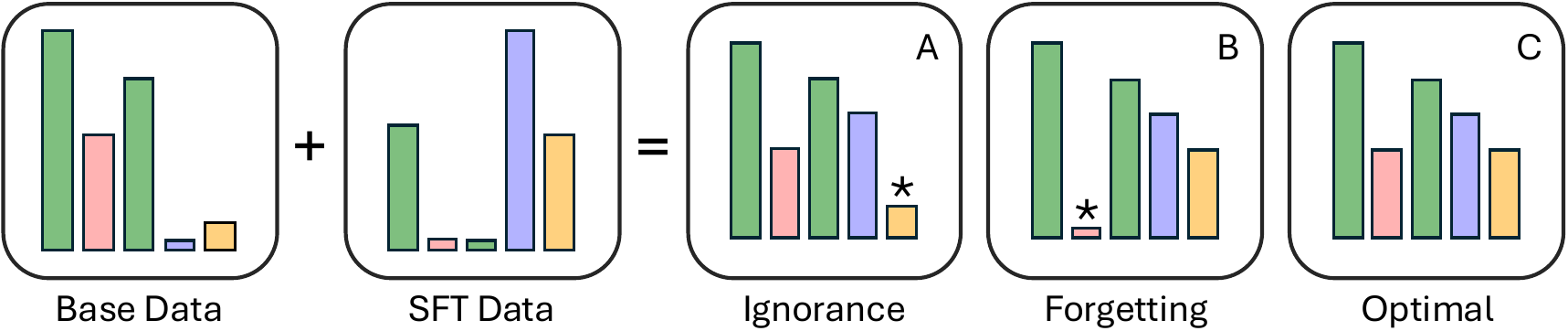}}
\caption{Impact of SFT on Generative Diversity. Comparison of data distributions and model states: Base Data represents the broad pretraining corpus, SFT Data represents the curated instruction set. The bars illustrate the discrete probability distribution over the vocabulary, color denotes token category. (A) \textbf{Ignorance} illustrates the failure to capture low-frequency SFT patterns, while (B) \textbf{Forgetting} depicts the erosion of the original pretraining knowledge. (C) \textbf{Optimal} represents the ideal balance where the model integrates new instructions without sacrificing the generative flexibility or richness of the base distribution.}
\label{fig:assumption}
\end{center}
\vskip -0.2in
\end{figure*}

However, SFT also introduces notable challenges. While pretrained LLMs naturally generate diverse outputs, fine-tuning is considered to reduce this variety, as has been observed for specific model families and evaluation sets \cite{omahony_attributing_2024}. This reduction is problematic, as high diversity offers significant practical advantages: it aids post-training alignment \cite{ouyang_training_2022}, facilitates RL exploration \cite{bai_constitutional_2022}, and enhances complex reasoning \cite{he_rethinking_2022,wang_self-consistency_2023}. Furthermore, diverse generation benefits multi-model systems through complementary generations \cite{brown_managing_2005} and provides users with a wider range of perspectives.

In an effort to preserve distributional breadth, prior work has explored various regularization techniques, including weight decay during fine-tuning \cite{krogh_simple_1991}, noise injection into input embeddings \cite{jain_neftune_2023}, adaptive sampling strategies \cite{troshin_control_2025,nguyen_turning_2025}, model ensembling \cite{hao_understanding_2025}, and alternative loss formulations \cite{li_preserving_2024, verine_improving_2025}. Although these methods offer various improvements, they are inspired by different analytical paradigms and do not fully resolve the underlying factors limiting generative diversity, leaving room for further refinement.

In this work, we attribute the reduction in output diversity during SFT to the interplay of two factors (Figure~\ref{fig:assumption}). First, fine-tuning on small, curated datasets often induces overfitting to dominant patterns \cite{bethune_scaling_2025}, causing the model to neglect the long-tail distribution of the data. Second, the process can erode the diverse knowledge acquired during pretraining \cite{goodfellow_empirical_2015}, further narrowing the model's available response space. The combination of these two phenomena leads to an overall diminished diversity. 

To counter these issues, we focus on loss-based regularization as a principled solution. We investigate existing SFT functions to determine their ability to preserve pretrained knowledge. In addition, we also evaluate Focal Loss (FL) \cite{lin_focal_2017} as a potential SFT objective for maintaining the balanced treatment of rare samples. Building on these insights, we propose a new Tempered Focal (TOFU) loss function explicitly designed to tackle both forgetting and ignorance simultaneously. Our main contributions can be summarized as follows:

\begin{enumerate} 
    \item We conduct an extensive evaluation of existing diversity-oriented SFT objectives across various model families, datasets, and benchmarks. 
    \item We investigate Focal Loss as an SFT objective for maintaining balanced treatment of rare samples and demonstrate its promising performance. 
    \item We propose a novel training objective, TOFU, designed to mitigate both ignorance and forgetting.
    \item We show that our method achieves superior generative diversity while maintaining competitive output quality compared to standard approaches. 
\end{enumerate}

\section{Theoretical Preliminaries and Backgrounds}
\paragraph{Large Language Models}
Large Language Models are trained as next-token predictors over a discrete vocabulary $\mathcal{V}$. Given a sequence of tokens $x_{1:L} = (x_1,...,x_L)$, the model defines a conditional distribution over the next token $x_l$ for each vocabulary element $y \in \mathcal{V}$:
\begin{equation}
    p_\theta(y | x_{<l}) \doteq p_\theta (x_l = y | x_{<l}).
\end{equation}
Training consists of minimizing the Cross-Entropy (CE) between the model’s predicted distribution and a target distribution $q \doteq q(y | x_{<l})$ at each position:
\begin{equation}
\label{def:ce}
    \mathcal{L}_{\mathrm{CE}}(\theta) = - \sum_{l=1}^L \mathbb{E}_q \log p_\theta(y | x_{<l})
\end{equation}
For standard supervised training, $q$ is one-hot on the ground-truth token, making the training equivalent to maximizing the likelihood of the observed sequence. For simplicity, throughout this paper we focus on the loss corresponding to a single token, without explicitly showing its dependence on the context, as this omission is purely notational and does not change the underlying mathematics.

\paragraph{Supervised Fine-Tuning} 
In Supervised Fine-Tuning, LLMs are adapted to specific tasks using sequences that combine a prompt and a response. The model conditions on the prompt as a fixed prefix but is optimized exclusively on the response. This ensures focus on generating the correct outputs for the task, while leveraging existing pretrained knowledge.

The standard training objective for SFT is the Cross-Entropy loss. It has a significant limitation: on relatively small fine-tuning datasets, it encourages the model to focus narrowly on the few observed responses while ignoring other plausible outputs \cite{li_preserving_2024}. Therefore, CE can reduce the expressiveness of the model’s generation, a property that is important for downstream exploration and robust alignment. To address this issue, recent work in the community has explored variants of SFT that employ modified loss functions specifically designed to preserve or enhance output diversity.

\paragraph{GEM by \citet{li_preserving_2024}}
Game-theoretic Entropy Maximization (GEM) reframes SFT as a distribution-matching process in which learning is modeled as transferring probability mass from non-target to target tokens. Instead of relying on Cross-Entropy, which forces indiscriminate and unbounded probability flow, GEM introduces a game-theoretic formulation. In this setup, a meta-controller regulates how and where probability mass moves. This selective control prevents collapse of the output distribution and avoids over-penalizing semantically meaningful or rare tokens. Ultimately, this approach yields a practical training algorithm in which the entire framework reduces to optimizing a new Cross-Entropy replacing loss function. 
\begin{definition}[GEM loss \cite{li_preserving_2024}]
\label{def:gem}
The GEM's objective is defined as follows:
\begin{equation}
\label{eq:gem}
\mathcal{L}_{\mathrm{GEM}}(\theta) = -  \mathbb{E}_{q} [\log p_\theta] + \mathbb{E}_{\red{p^\beta_\theta}} [\log p_\theta],
\end{equation}
where $p^\beta_\theta \doteq \text{softmax}\left(\beta^{-1} \log p_\theta\right)$ is a temperature-scaled distribution with $\beta \in (0,1)$ as a temperature parameter. Note that $\red{p^\beta_\theta}$ is detached from gradients computation. Here and throughout this work we highlight detached gradients in \red{red}.
\end{definition}

Taken together, GEM produces sparse, targeted updates that preserve useful pretraining knowledge while still aligning the model to the supervised dataset. The reduced forgetting results in greater output diversity, as the model maintains a broader and more balanced token distribution.

\paragraph{$\lambda$-PR by \citet{verine_improving_2025}}
$\lambda$-PR is a training objective that explicitly manages the trade-off between quality and diversity, building on previous approaches that improve generation. 

\begin{definition}[$\lambda$-PR loss \cite{verine_improving_2025}]
\label{def:pr}
The $\lambda$-PR's objective is defined as follows:
\begin{equation}
\label{eq:pr}
   \mathcal{L}_{\lambda-\mathrm{PR}}(\theta) = - \mathbb{E}_{q} \left [ w(\lambda, \alpha) \log p_{\theta}  \right ],
\end{equation}
where $w(\lambda, \alpha) = \lambda^{\frac{l-1}{L}} \mathbb{I}_{\red{p_\theta} \leq \delta} \frac{\red{{p_\theta}}}{\alpha+(1-\alpha)\red{p_\theta}}$. Here, $\lambda \in \mathbb{R}^+$ is the main parameter controlling a trade-off, $\alpha \in [0,1]$ and $\delta = \frac{\alpha\lambda^{1/L}}{1-(1-\alpha)\lambda^{1/L}}$ ($L$ being sequence length and $l$ token's position).
\end{definition}

It draws on ideas such as removing high-loss examples to reduce the impact of noisy references \cite{kang_improved_2020}, downweighting unlikely sequences to prevent degeneration \cite{ji_tailoring_2023}, and reweighting gradients to promote high-probability outputs \cite{pang_text_2021}. Although these techniques were originally intended to enhance quality, $\lambda$-PR partially reverses their effects to optimize for diversity. The resulting compound loss function provides a unified framework for fine-tuning the balance between the two competing metrics. $\lambda$-PR can be interpreted as a weighted Cross-Entropy loss, where different tokens contribute unequally to the overall objective. The loss is motivated by two complementary goals: removing low log-loss examples to encourage a broader range of outputs, and downweighting less probable tokens to preserve accuracy.

\paragraph{Focal Loss by \citet{lin_focal_2017}}
Focal Loss, previously widely used in computer vision for unbalanced training \citep{lin_focal_2017}, has been recently explored for enhancing the quality of large language models \cite{rege_cambrin_beyond_2024,xia_influences_2025}, yet its potential remains relatively overlooked.  It addresses class imbalance by reshaping the standard Cross-Entropy loss to downweight well-classified examples.

\begin{definition}[Focal Loss \cite{lin_focal_2017}]
\label{def:focal}
The Focal Loss is defined as follows:
\begin{equation}
   \mathcal{L}_{\mathrm{FL}}(\theta) = - \mathbb{E}_{q} \left [ (1-p_{\theta})^\gamma \log p_{\theta}  \right ],
\end{equation}
where power coefficient $\gamma \geq 0$ is typically selected from range $[1, 5]$ \cite{lin_focal_2017, mukhoti_calibrating_2020, charoenphakdee_focal_2020}.
\end{definition}

We hypothesize that this loss can enhance the model's diversity by placing greater emphasis on underrepresented examples in the SFT dataset.

\section{Proposed Method}

Forgetting and ignorance arise from different aspects of the training dynamics. We propose to address both by combining the forgetting mitigation of GEM with the weighting of Focal Loss for underrepresented samples.
However, combining the objectives is not straightforward, as adding terms or factors can significantly alter the update structure, leading to unstable optimization or poor interpretability. This motivates a careful analysis at the gradient level.

\begin{theorem}[GEM loss equivalence]
    \label{th:gem}
    Training with $\mathcal{L}_{\mathrm{GEM}}$ is equivalent to training with temperature-scaled Cross-Entropy loss, as 
    \begin{equation}
        \nabla_\theta \mathcal{L}_{\mathrm{GEM}}(\theta) = \nabla_\theta \mathcal{L}^\beta_{\mathrm{CE}}(\theta),
    \end{equation}
    where $\mathcal{L}^\beta_{\mathrm{CE}}(\theta) \doteq -\beta \mathbb{E}_{q} \log p^\beta_\theta$.
\end{theorem}

By examining GEM's gradients, we gain both theoretical insight and practical benefits. Specifically, our analysis reveals that GEM is equivalent to a temperature-scaled Cross-Entropy loss (Theorem \ref{th:gem}, full proof is given in Appendix~\ref{th_app:gem}). The equivalence clarifies GEM’s advantage over methods that explicitly relax the predictive distribution by adding an entropy term. Instead of uniform flattening, GEM exerts softer pressure via adaptation, preventing the model from becoming overly confident while still guiding it toward the target. For practitioners, this equivalence also simplifies the computation required to implement GEM. 

Theoretically, the GEM objective shares a global minimum with Cross-Entropy when the latter is scaled by $\beta^{-1}$ after fine-tuning. In practice, however, the model does not converge to that point due to the limited updates in SFT and inherent optimization noise. This allows GEM to reach a parameter solution distinct from that of standard post-hoc temperature scaling \cite{ficler_controlling_2017}.

We next analyze Focal Loss in a similar manner. Prior work \cite{mukhoti_calibrating_2020} has shown that despite not being a proper loss it can be interpreted as a gradient-scaled Cross-Entropy. For completeness, we restate this result and explicitly state the assumptions required for the equivalence in Proposition~\ref{pr:focal} (the full proof is given in Appendix~\ref{pr_app:focal}).

\begin{proposition}[Focal Loss equivalence for one-hot targets]
    \label{pr:focal}
    Assume that the target distribution $q$ is one-hot. Then, for Focal Loss $\mathcal{L}_{\mathrm{FL}}$ and Cross-Entropy $\mathcal{L}_{\mathrm{CE}}$, the gradients satisfy
    \begin{equation}
        \nabla_\theta \mathcal{L}_{\mathrm{FL}}(\theta) = g(\hat{p}_\theta,\gamma) \nabla_\theta \mathcal{L}_{\mathrm{CE}}(\theta),
    \end{equation}
    where $g(p, \gamma) = (1-p)^\gamma - \gamma p (1-p)^{\gamma-1} \log p$. Here and throughout this work $\hat{p}$ denotes the predicted probability assigned to the ground-truth token.
\end{proposition}

The scaling function $g(p, \gamma)$ increases the weight corresponding to the underrepresented samples (Figure \ref{fig:coeffs}). By peaking at a small but non-zero $p$, the gradient magnitude prioritizes moderately difficult tokens relative to extremely hard ones. This behavior echoes the idea of downweighting based on predicted probability usually employed for quality enhancement \cite{kang_improved_2020,pang_text_2021,ji_tailoring_2023}.

Comparing FL scaling with that of the $\lambda$-PR (Figure \ref{fig:coeffs}), one might notice the latter's downside: it cannot account for a previously zeroed-out probability $p$ without assigning it the highest possible weight. Consequently, the $\gamma$ parameter of FL provides more flexibility in steering the weighting balance than the $\delta$ of $\lambda$-PR.

\begin{figure}[t]
\vskip 0.2in
\begin{center}
\centerline{\includegraphics[width=\linewidth]{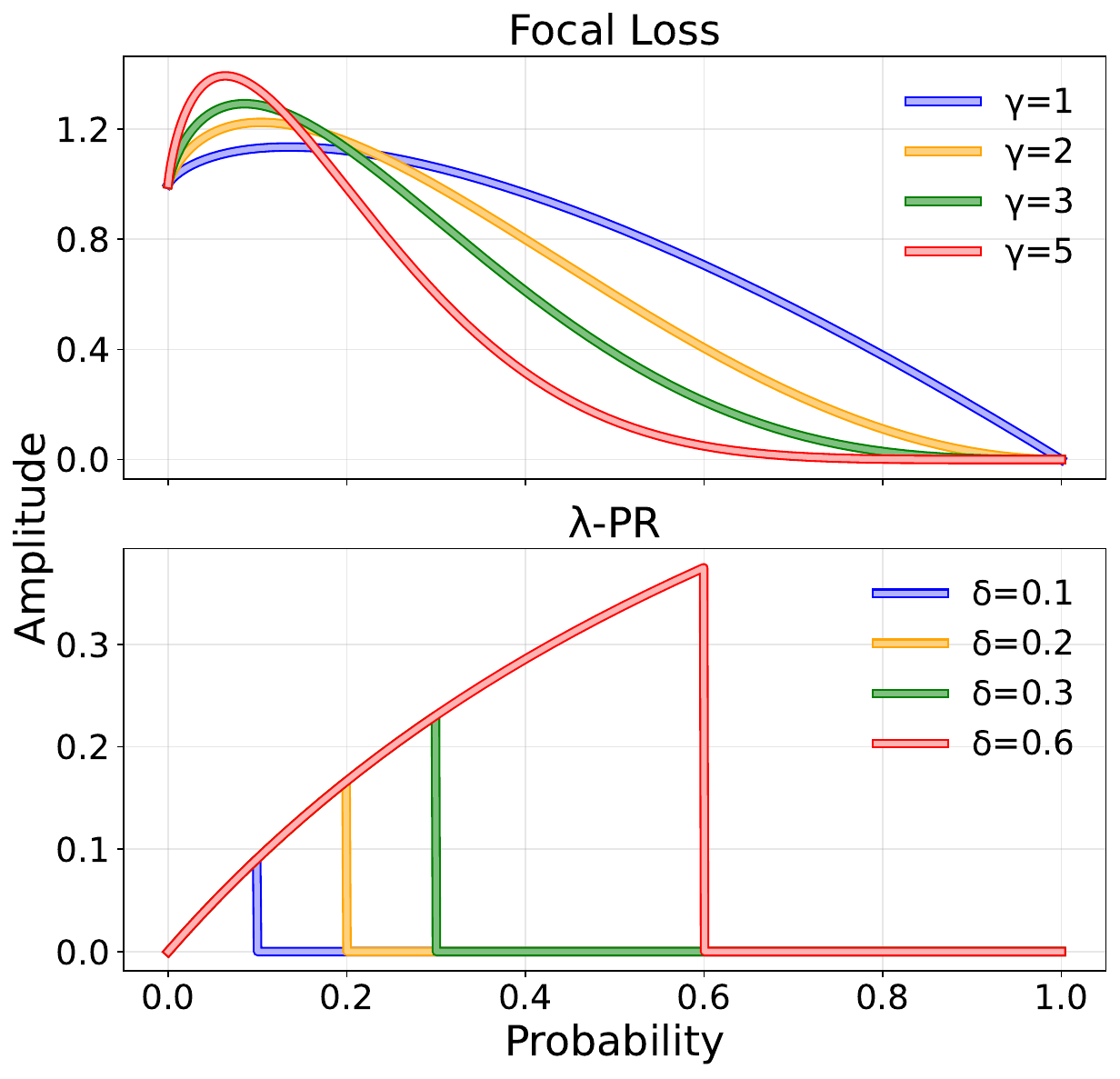}}
\caption{Gradient scaling amplitudes for Focal Loss and $\lambda$-PR as functions of various parameters and probabilities $p$. For $\lambda$-PR we omit the dependence on tokens position for simplicity. 
}
\label{fig:coeffs}
\end{center}
\vskip -0.2in
\end{figure}

\begin{corollary}
    \label{cor:focalb}
    If $q$ is one-hot, then weighting $\mathcal{L}^\beta_{\mathrm{CE}}(\theta)$ by the focal term $(1-p_\theta^\beta)^\gamma$ results in a gradient of the form $g(\hat{p}_\theta^\beta, \gamma) \nabla_\theta \mathcal{L}^\beta_\mathrm{CE}(\theta)$.
\end{corollary}

The equivalences established by Theorem \ref{th:gem} and Proposition \ref{pr:focal} allow us to merge the core ideas of GEM and Focal Loss by applying the focal term to the temperature-scaled Cross-Entropy. However, as follows from Corollary \ref{cor:focalb} (the complete proof is given in Appendix~\ref{cor_app:focalb}), a naive multiplication of $\mathcal{L}^\beta_{\mathrm{CE}}(\theta)$ by the focal term can excessively amplify the influence of very small probabilities in gradient weighting. This occurs because temperature-scaling reduces these probabilities further, which in turn increases the corresponding focal weight, potentially leading to unintended overemphasis on hard examples. Ideally, one would like to scale the GEM gradients using the focal scaling function $g(p,\gamma)$ evaluated on the unscaled probabilities $p$. Simply multiplying the temperature-scaled Cross-Entropy loss by a focal term that depends on $p$ does not achieve this, because the gradients would propagate differently through $p$ and $p^\beta$. 
The desired effect can be attained by using a detached version of $g(p,\gamma)$, which ensures that the gradient is scaled correctly.
Therefore, in this work, we present Tempered Focal (TOFU) training objective, which naturally combines the advantages of Focal Loss and GEM.  

\begin{definition}[TOFU loss]
\label{def:tony}
Let us denote $p^\beta$ as a temperature-scaled distribution (\ref{def:gem}), then
\begin{equation}
\mathcal{L}_{\mathrm{TOFU}}(\theta) = - \mathbb{E}_{q} \left [ \red{g(\hat{p}_\theta,\gamma)} \beta \log p^\beta_{\theta}  \right ],
\end{equation}
where $g(\hat{p}_\theta,\gamma)$ is detached from gradient computation.
\end{definition}

Due to the inclusion of the focal term, TOFU does not directly optimize a well-defined statistic (unlike Cross-Entropy, which minimizes the Kullback–Leibler divergence). 
\par
\begin{corollary}
    \label{cor:tony}
    If the target distribution $q$ is one-hot, TOFU loss gradient is proportional to the one of the temperature-scaled CE: 
    \begin{gather}
        \nabla_\theta \mathcal{L}_{\mathrm{TOFU}}(\theta) = g(\hat{p}_\theta,\gamma)\nabla_\theta \mathcal{L}^\beta_{\mathrm{CE}}(\theta)
    \end{gather}
\end{corollary}

Nevertheless, it can be interpreted as a form of adaptive gradient scaling of the temperature-scaled Cross-Entropy loss, emphasizing under-predicted tokens while preventing the learned distribution from becoming overly concentrated (Corollary~\ref{cor:tony}, proof is given in Appendix~\ref{cor_app:tony}). As TOFU interpretability is restricted to one-hot target distributions, its applicability to specific SFT distillation remains tenuous, a limitation that similarly extends to the use of Focal Loss.

\section{Experimental Evaluation}

\subsection{Tested LLMs}

To properly evaluate our method and its alternatives, we aim for an experimental setup that is both realistic and diverse in terms of the models tested. In practice, widely used families such as Llama-3 \cite{grattafiori_llama_2024}, Phi-4 \cite{abdin_phi-4_2024}, and Qwen-3 \cite{yang_qwen3_2025} suggest an underlying post-training process that results in security guardrails, though their exact development pipelines remain largely opaque. The potential of alignment to obscure the isolated effects of SFT necessitates a more careful evaluation design. Therefore, in addition to Llama-3.1-8B, Qwen-3-8B, and Phi-4-14B, we focus on the models that have completed only the pretraining stage: OLMo-2-13B \cite{groeneveld_olmo_2024}, Mistral-12B \cite{mistral_ai_mistral_2024}, and Pythia-12B \cite{biderman_pythia_2023}. Although the resulting selection is expected to perform well in simple settings, complex reasoning requires specific training. To meet this requirement, we have added to our setup Qwen-2.5-Math-1.5B, Qwen-2.5-Math-7B \cite{yang_qwen25-math_2024} and DeepSeek-Math-7B \cite{shao_deepseekmath_2024}. As such, we are able to evaluate the SFT objectives both in isolation and in combination with other post-training procedures on a wide range of tasks.

\begin{table*}[t]
\caption{Performance of models across Alpaca SFT objectives on Short Stories and Small Prompts. Diversity (D) is measured via Self-BLEU (0–100), where lower scores are better. Quality (Q) is measured via LLM Judge score (0–5), where higher scores are better.} 
\setlength{\tabcolsep}{1.1pt}
\label{tab:alpacasssp}
\begin{center}
\begin{small}
\begin{sc}
\begin{tabular}{clcccccccccccc}
\toprule
Bench & Method & \multicolumn{2}{c}{Mistral-12B} & \multicolumn{2}{c}{OLMo-2-13B} & \multicolumn{2}{c}{Pythia-12B} & \multicolumn{2}{c}{Llama-3.1-8B} & \multicolumn{2}{c}{Qwen-3-8B} & \multicolumn{2}{c}{Phi-4-14B} \\

& & D$\downarrow$ & Q$\uparrow$ & D$\downarrow$ & Q$\uparrow$ & D$\downarrow$ & Q$\uparrow$ & D$\downarrow$ & Q$\uparrow$ & D$\downarrow$ & Q$\uparrow$ & D$\downarrow$ & Q$\uparrow$ \\

\midrule

\multirow{6}{*}[-2.5ex]{SS} & Base & 11.4$_{\pm 4.7}$ & 3.9$_{\pm 0.5}$ & 12.3$_{\pm 4.6}$ & 3.7$_{\pm 0.7}$ & 9.3$_{\pm 3.1}$ & 2.9$_{\pm 0.6}$ & 11.5$_{\pm 6.1}$ & 3.5$_{\pm 0.7}$ & 24.9$_{\pm 8.9}$ & 2.9$_{\pm 1.0}$ & 13.8$_{\pm 6.8}$ & 3.4$_{\pm 0.9}$ \\[1ex]

& CE & 22.5$_{\pm 8.3}$ & 4.8$_{\pm 0.2}$ & 24.5$_{\pm 9.1}$ & 4.8$_{\pm 0.2}$ & 23.8$_{\pm 8.4}$ & 3.7$_{\pm 0.6}$ & 22.6$_{\pm 8.5}$ & 4.7$_{\pm 0.2}$ & 22.3$_{\pm 7.7}$ & 4.6$_{\pm 0.3}$ & 24.8$_{\pm 9.0}$ & 4.8$_{\pm 0.2}$ \\

\cmidrule(){2-14}

& \textcolor{grey}{$\lambda$-PR} & \textcolor{grey}{3.7$_{\pm 0.7}$} & \textcolor{grey}{2.5$_{\pm 0.5}$} & \textcolor{grey}{3.9$_{\pm 0.6}$} & \textcolor{grey}{3.0$_{\pm 0.4}$} & \textcolor{grey}{4.1$_{\pm 0.7}$} & \textcolor{grey}{2.1$_{\pm 0.5}$} & \textcolor{grey}{3.8$_{\pm 0.7}$} & \textcolor{grey}{2.7$_{\pm 0.5}$} & \textcolor{grey}{5.0$_{\pm 1.2}$} & \textcolor{grey}{3.4$_{\pm 0.4}$} & \textcolor{grey}{4.0$_{\pm 0.6}$} & \textcolor{grey}{3.1$_{\pm 0.4}$} \\[1ex]

& FL & 16.8$_{\pm 5.4}$ & 4.7$_{\pm 0.2}$ & 15.9$_{\pm 5.4}$ & 4.7$_{\pm 0.2}$ & 14.6$_{\pm 4.3}$ & 3.8$_{\pm 0.6}$ & 14.3$_{\pm 4.7}$ & 4.5$_{\pm 0.3}$ & 16.4$_{\pm 6.2}$ & 4.6$_{\pm 0.3}$ & 16.2$_{\pm 5.4}$ & 4.7$_{\pm 0.3}$ \\[1ex]

& GEM & 14.3$_{\pm 6.0}$ & 4.6$_{\pm 0.3}$ & 13.2$_{\pm 5.0}$ & 4.6$_{\pm 0.3}$ & 13.0$_{\pm 6.2}$ & 3.4$_{\pm 0.6}$ & 11.5$_{\pm 4.0}$ & 4.5$_{\pm 0.3}$ & 13.9$_{\pm 4.6}$ & 4.5$_{\pm 0.3}$ & 13.4$_{\pm 5.1}$ & 4.7$_{\pm 0.3}$ \\[1ex]

& TOFU & \textbf{12.7$_{\pm 4.7}$} &  4.6$_{\pm 0.3}$ & \textbf{11.9$_{\pm 4.1}$} & 4.6$_{\pm 0.3}$ & \textbf{10.8$_{\pm 3.7}$} & 3.6$_{\pm 0.5}$ & \textbf{11.2$_{\pm 4.2}$} & 4.5$_{\pm 0.3}$ & \textbf{12.7$_{\pm 3.8}$} & 4.5$_{\pm 0.3}$ & \textbf{12.4$_{\pm 4.1}$} & 4.6$_{\pm 0.3}$ \\

\midrule

\multirow{6}{*}[-2.5ex]{SP} & Base & 12.7$_{\pm 7.0}$ & 3.8$_{\pm 0.9}$ & 13.9$_{\pm 7.2}$ & 3.8$_{\pm 1.0}$ & 8.4$_{\pm 3.0}$ & 2.6$_{\pm 1.0}$ &12.2$_{\pm 6.0}$ & 3.5$_{\pm 0.9}$ & 31.5$_{\pm 12.5}$ & 3.8$_{\pm 0.9}$ & 17.6$_{\pm 9.3}$ & 3.9$_{\pm 0.9}$ \\[1ex]

&  CE & 44.5$_{\pm 14.5}$ & 4.2$_{\pm 0.7}$ & 45.0$_{\pm 14.6}$ & 4.3$_{\pm 0.7}$ & 35.9$_{\pm 14.4}$ & 3.8$_{\pm 0.8}$ & 44.5$_{\pm 14.4}$ & 4.2$_{\pm 0.7}$ & 44.9$_{\pm 13.4}$ & 4.1$_{\pm 0.7}$ & 46.6$_{\pm 15.0}$ & 4.2$_{\pm 0.7}$ \\

\cmidrule(){2-14}

& \textcolor{grey}{$\lambda$-PR} & \textcolor{grey}{2.9$_{\pm 0.8}$}& \textcolor{grey}{2.0$_{\pm 0.7}$} & \textcolor{grey}{3.2$_{\pm 0.9}$} & \textcolor{grey}{2.3$_{\pm 0.7}$} & \textcolor{grey}{3.0$_{\pm 1.1}$} & \textcolor{grey}{2.0$_{\pm 0.6}$} & \textcolor{grey}{2.8$_{\pm 0.8}$} & \textcolor{grey}{2.1$_{\pm 0.7}$} & \textcolor{grey}{4.6$_{\pm 2.1}$} & \textcolor{grey}{2.6$_{\pm 0.7}$} & \textcolor{grey}{3.3$_{\pm 1.0}$} & \textcolor{grey}{2.4$_{\pm 0.6}$} \\[1ex]

& FL & 29.3$_{\pm 11.1}$ & 4.1$_{\pm 0.7}$ & 29.0$_{\pm 9.8}$ & 4.2$_{\pm 0.7}$ & 22.2$_{\pm 9.8}$ & 3.6$_{\pm 0.8}$ & 28.3$_{\pm 10.5}$ & 4.1$_{\pm 0.7}$ & 29.2$_{\pm 10.6}$ & 4.0$_{\pm 0.7}$ & 28.9$_{\pm 10.4}$ & 4.2$_{\pm 0.7}$ \\[1ex]

& GEM & 27.7$_{\pm 12.6}$ & 4.1$_{\pm 0.7}$ & 27.0$_{\pm 11.6}$ & 4.1$_{\pm 0.7}$ & 17.6$_{\pm 7.6}$ & 3.5$_{\pm 0.7}$ & 25.9$_{\pm 12.3}$ & 4.1$_{\pm 0.7}$ & 29.5$_{\pm 11.5}$ & 4.0$_{\pm 0.8}$ & 27.5$_{\pm 12.6}$ & 4.1$_{\pm 0.7}$ \\[1ex]

& TOFU & \textbf{21.2$_{\pm 8.8}$} & 4.0$_{\pm 0.6}$ & \textbf{21.3$_{\pm 8.6}$} & 4.1$_{\pm 0.6}$ & \textbf{16.1$_{\pm 7.6}$} & 3.5$_{\pm 0.8}$ & \textbf{20.9$_{\pm 8.4}$} & 4.0$_{\pm 0.7}$ & \textbf{24.1$_{\pm 8.9}$} & 4.0$_{\pm 0.7}$ & \textbf{21.8$_{\pm 8.7}$} & 4.1$_{\pm 0.6}$ \\

\bottomrule
\end{tabular}
\end{sc}
\end{small}
\end{center}
\vskip -0.1in
\end{table*}

\subsection{SFT Setup}

To fine-tune the selected models, we used the Alpaca instruction dataset \cite{taori_stanford_2023}, a widely adopted collection of diverse instructions paired with demonstrations. As an alternative, we also experimented with the UltraFeedback \cite{cui_ultrafeedback_2023} alignment dataset. To fine-tune reasoning models in a Chain-of-Thought (CoT) manner, we sampled 100, 000 problems from NuminaMath-CoT dataset \cite{numina}. 

We performed fine-tuning using the QLoRA framework \cite{dettmers_qlora_2023}, which applies Low-Rank Adaptation \cite{hu_lora_2021} to models quantized with 4-bit NormalFloat (NF4). This approach significantly reduces memory footprint and accelerates training without a substantial loss in performance. More technical details can be found in Appendix \ref{appendix:details}.

\subsection{Benchmarks}

\paragraph{Instruction Following} As our primary goal is to investigate how diversity is affected by SFT under different objectives, we prioritize benchmarks containing open-ended prompts that favor creative synthesis over deterministic accuracy. The more straightforward tasks, such as story continuation and constrained generation (instruction following), are represented by Short Stories (SS) and Small Prompts (SP), respectively. We composed both datasets from open-source materials, with the curation process detailed in Appendix~\ref{appendix:details}. Further challenging the models, we employ NoveltyBench (NB)~\cite{zhang_noveltybench_2025}. It is designed to measure the capacity for generating multiple distinct high-quality outputs using specifically curated prompts to elicit diverse responses. 

\paragraph{Reasoning} We also examine whether improved diversity aids in solving complex reasoning tasks by enabling a broader search for correct solutions across a suite of mathematical benchmarks: MATH500 \cite{hendrycks_measuring_2021}, GSM8K \cite{cobbe_training_2021} and MinervaMath \cite{lewkowycz_solving_2022}. The tested ability is particularly relevant if the goal of SFT is to enhance exploration without introducing low-quality noise, providing a better foundation for the subsequent Reinforcement Learning phase.

\paragraph{Factuality} One might be concerned that prioritizing output diversity could degrade factual accuracy. To monitor this trade-off, we evaluated the models on standard multiple-choice benchmarks for retrieval of professional and scientific knowledge, ARC \cite{clark_think_2018} and MMLU \cite{hendrycks_measuring_2020}. As our instruction tuning utilizes general-purpose datasets, these tasks primarily measure the retention of information from the initial training phase.

\paragraph{Safety}  Increased diversity may reduce sensitivity to malicious prompts by broadening the range of possible responses, including unsafe ones. To investigate this topic, we adopted two commonly used red-teaming benchmarks, specifically Malicious Instruct \cite{huang_catastrophic_2023} and HarmBench \cite{mazeika_harmbench_2024}, and evaluated the model's robustness to adversarial and harmful instructions. 

\subsection{Metrics}

We evaluate output diversity using several complementary approaches across the benchmarks. On the Short Stories and Small Prompts, we employ the widely adopted Self-BLEU metric \cite{zhu_texygen_2018}, which quantifies surface-level variation by measuring word and phrase overlap across responses for a given prompt. NoveltyBench provides its own diversity metric, \textit{Distinct}, employing a specialized LLM based classifier, that prioritizes semantic variation over surface-level linguistics. Consequently, it captures meaningful distinctions between outputs that standard metrics may overlook.

To assess the quality of the generations for SS and SP, we use an LLM-as-a-judge approach \cite{liu_g-eval_2023}, scoring responses based on coherence, consistency, and fluency (Appendix~\ref{app:judge}). Similarly, NoveltyBench proposes the \textit{Utility} metric and its associated LLM usefulness judge. 

For reasoning tasks, we evaluate exploration capabilities by calculating coverage (Pass@k), defined as the probability that at least one generated solution converges to the correct answer. We also track the mean success rate to distinguish between two behaviors: whether a model simply generates correct solutions more frequently (higher precision) or whether its increased diversity allows it to discover correct solutions for harder problems \cite{cobbe_training_2021}. In the latter case, the coverage would increase even if the mean success rate remains stagnant. To calculate both coverage and mean success scores, we extracted the final answers provided by the model in bounding box format and compared them against the ground truth.

For factual tasks, we decide on the output correctness directly comparing it with the reference answer. To evaluate safety alignment, we employ an Attack Success Rate (ASR) \cite{zou_universal_2023}, the percentage of instructions that receive misaligned outputs, defined as failing to abstain from responding to a malicious instruction.

\section{Results}

To investigate the performance of TOFU, we evaluate it along with the standard Cross-Entropy and the diversity-oriented objectives GEM and $\lambda$-PR. To isolate the impact of prioritizing infrequent
patterns, we also benchmark Focal Loss. Hyperparameter settings for GEM and $\lambda$-PR match their original publications, whereas the configurations for Focal Loss and TOFU were derived from ablation studies (see Appendix~\ref{appendix:ablations}). Notably, TOFU's parameters were selected only once, on the ARC dataset, and remained fixed across all subsequent experiments, demonstrating the method's robustness.

\paragraph{Creative Writing \& Instruction Following}

We first explore models diversity on Short Stories and Small Prompts. The results for Alpaca as SFT dataset are reported in Table~\ref{tab:alpacasssp}. On both benchmarks, Cross-Entropy increases quality but shows a significant reduction in diversity compared to the base model, confirming the problem addressed in this work. The only exception is the Qwen-3-8B model when evaluated on Short Stories. Its post-training procedures appear to have a significant residual effect: even in the absence of an instruction template, the model consistently attempts to engage in Chain-of-Thought reasoning. This case highlights the importance of employing unaligned models for SFT-related evaluations to avoid the confounding effects of prior alignment.

Beyond the initial baseline, our evaluation reveals a clear hierarchy of objectives performance. Focal Loss yields noticeably higher diversity than Cross-Entropy, though this gain sometimes comes at a slight cost to quality. This validates our hypothesis that diversity is negatively affected by the ignorance of underrepresented training samples during the SFT process. TOFU maintains quality on par with GEM while achieving superior diversity across all the models and therefore emerges as the prominent choice for these benchmarks. 

While $\lambda$-PR reaches the absolute highest diversity, it does so at the cost of significant quality degradation, pushing the method beyond the bounds of usability. Consequently, to maintain a focused and computationally efficient analysis, we decided to limit the use of $\lambda$-PR in further evaluation.

\begin{figure}[t]
  \vskip 0.2in
  \begin{center}
    \centerline{\includegraphics[width=\columnwidth]{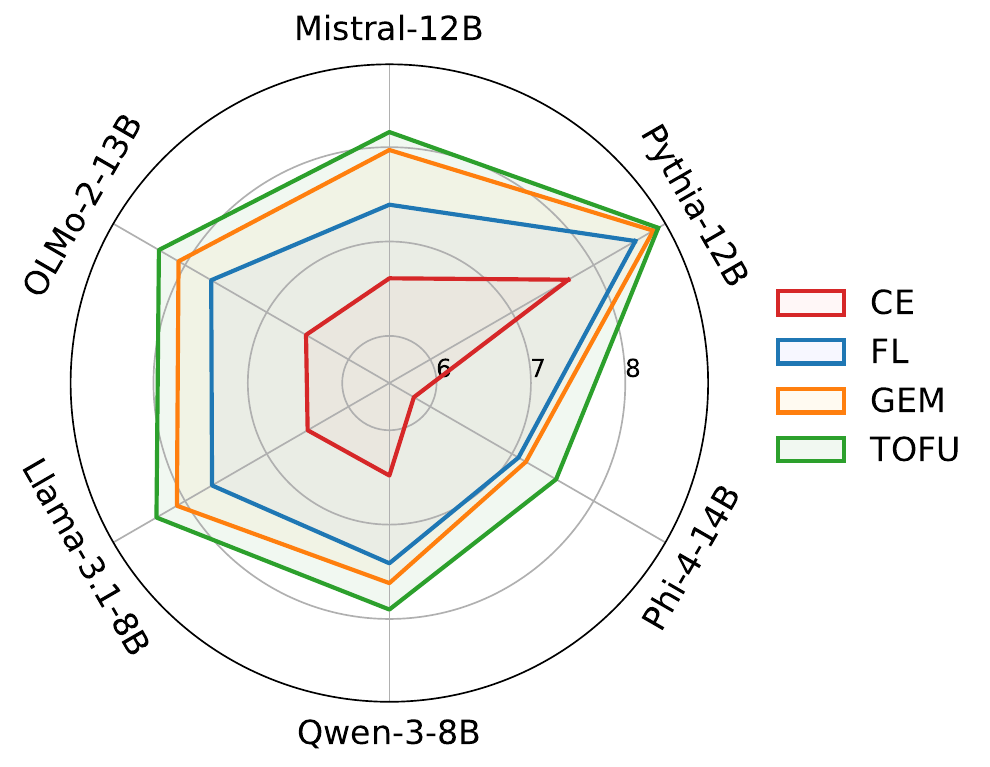}}
    \caption{
    The results on NoveltyBench dataset across different models and methods. Plotted values represent Distinct (1–10), measuring responses diversity, higher values indicate superior performance.
    }
    \label{fig:nb}
  \end{center}
\end{figure}
  
\begin{figure}[t]
  \vskip 0.2in
  \begin{center}
    \centerline{\includegraphics[width=\columnwidth]{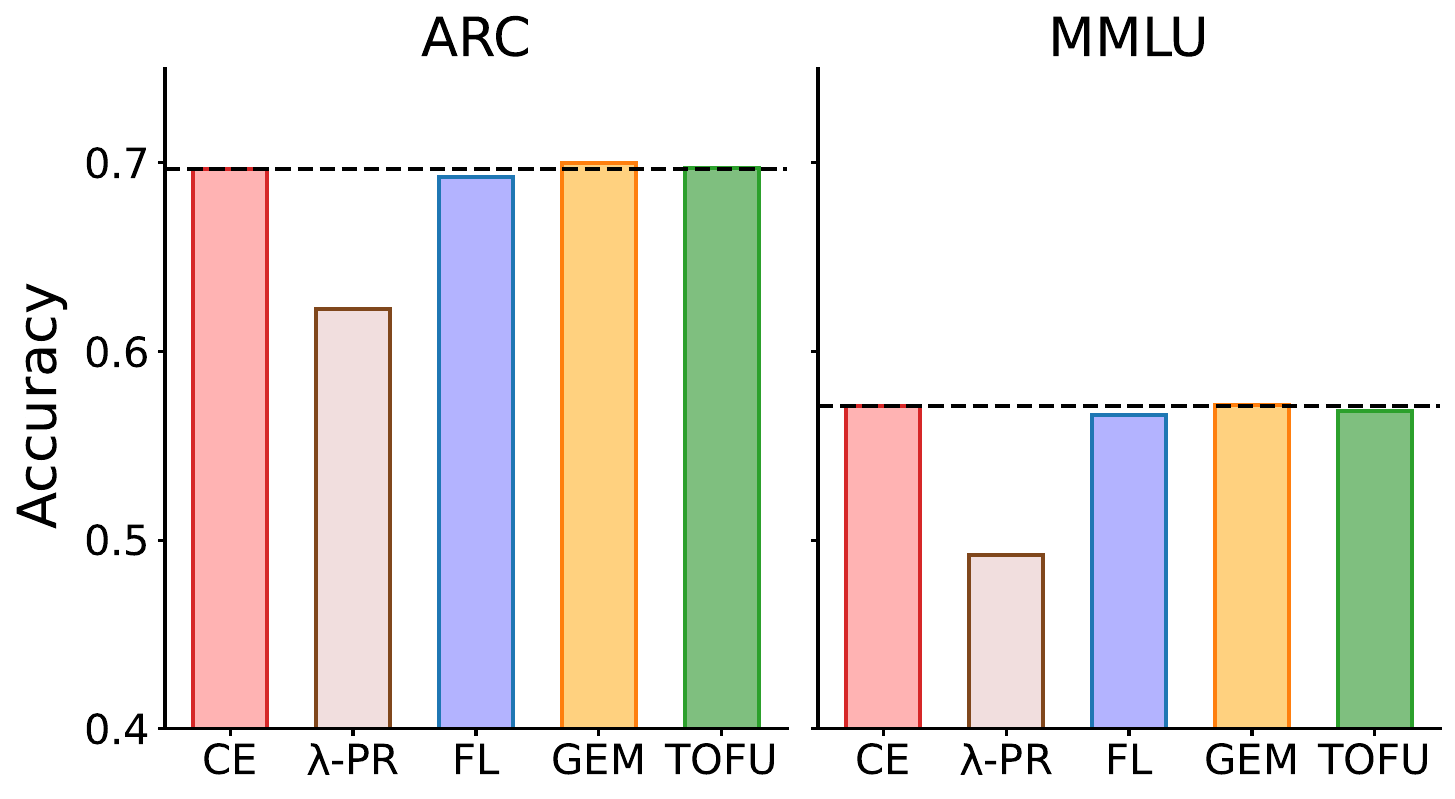}}
    \caption{
    Evaluation results of fine-tuning methods for the ARC and MMLU benchmarks averaged across the tested models. 
    The values represent the mean accuracy scores (0-1) across all tasks in each benchmark. The dotted line serves as a reference point for CE performance. 
    }
    \label{fig:arcmmlu}
  \end{center}
\end{figure}

The performance hierarchy on the NoveltyBench dataset remains largely consistent with our previous findings, as illustrated in Figure~\ref{fig:nb}. Specifically, TOFU consistently outperforms other objectives across all model families by demonstrating a superior Distinct score. Notably, quality remains comparable regardless of the model architecture, given the high Utility variance (Appendix~\ref{appendix:additional}). Ultimately, the results establish TOFU as the most effective diversity-oriented approach across both structured instruction and open-ended creative tasks. Additional experiments with UltraFeedback as an alternative SFT dataset further solidify this conclusion (Appendix~\ref{appendix:additional}).

\begin{table*}[t]
\caption{Performance of math reasoning models fine-tuned with different objectives on math datasets. The best values for the coverage and the average success rate (in parentheses) are highlighted in bold.}
\label{tab:mathreasoning}
\begin{center}
\begin{small}
\begin{sc}
\begin{tabular}{llccc}
\toprule
Bench & Method & Qwen-2.5-Math-1.5B & Qwen-2.5-Math-7B & DeepSeek-Math-7B \\
\midrule
\multirow{3}{*}{MATH500} & CE & 78.4 (\textbf{53.9}) & 84.8 (54.2) & 71.4 (32.7) \\
& GEM & 78.4 (53.4) & 83.2 (\textbf{59.0}) & 72.2 (\textbf{36.4}) \\
& TOFU & \textbf{80.6} (50.3) & \textbf{86.0} (53.8) & \textbf{72.6} (33.0) \\
\midrule
\multirow{3}{*}{Minerva} & CE & 28.3 (12.1) & 30.9 (\textbf{15.1}) & 34.2 (\textbf{12.5}) \\
& GEM & 31.3 (\textbf{12.4}) & 34.9 (\textbf{15.1}) & 35.3 (\textbf{12.5}) \\
& TOFU & \textbf{33.5} (10.7) & \textbf{38.6} (14.7) & \textbf{39.3} (12.0) \\
\midrule
\multirow{3}{*}{GSM8K} & CE & 79.2 (54.5) & 81.7 (60.1) & 76.0 (48.0) \\
& GEM & 73.2 (\textbf{55.5}) & 76.2 (\textbf{61.5}) & 72.4 (\textbf{51.2}) \\
& TOFU & \textbf{80.6} (51.8) & \textbf{88.0} (46.8) & \textbf{79.0} (47.5) \\
\bottomrule
\end{tabular}
\end{sc}
\end{small}
\end{center}
\vskip -0.1in
\end{table*}
\paragraph{Mathematical Reasoning}

Our results on mathematical reasoning, summarized in Table~\ref{tab:mathreasoning}, demonstrate that TOFU models achieve a visibly higher coverage than the Cross-Entropy and GEM, while average success rates remain stagnant or slightly decline. This observation suggests that TOFU does not increase per-sample correctness, but instead promotes more diverse exploration, leading to a higher probability of finding a correct solution. 

\paragraph{Factuality}
Regarding the factuality tasks, we observe no evidence of catastrophic forgetting, with performance remaining largely comparable to CE models. The sole exception is $\lambda$-PR, which exhibits a consistent performance degradation. In Figure~\ref{fig:arcmmlu} we report accuracies after Alpaca SFT averaged across all models, while the full results are available in Appendix~\ref{appendix:additional}. Overall, TOFU high diversity does not come at the cost of factual knowledge.

\paragraph{Safety}
To demonstrate that the diversity gains do not compromise models safety, we categorized models into two groups: \textit{safe} models (Qwen-3 and Phi-4), which express inherent security guardrails, and \textit{unsafe} models, which are much less sensitive to malicious prompts. Figure~\ref{fig:safety} reports the average Attack Success Rate (ASR) for both groups on Malicious Instruct, while detailed results on both Malicious Instruct and HarmBench are provided in Appendix~\ref{appendix:additional}. Our results indicate that for the unsafe model group, safety metric appears to have reached a saturation point where further degradation is negligible. Conversely, for safe models, TOFU does not lead to an increase in ASR, demonstrating that the method avoids introducing new vulnerabilities to jailbreak attempts.

\section{Discussions}

In this work, we have presented an extensive evaluation of SFT objectives across a range of model families, datasets, and benchmarks. Through a systematic analysis we provide robust confirmation that standard Cross-Entropy SFT consistently reduces diversity. To our best knowledge, this is the first study to systematically validate this trend across a broad range of model scales and diverse SFT datasets, as previous work of \citet{omahony_attributing_2024} primarily explored this hypothesis within the Pythia family at smaller sizes. By extending the evaluation to these more substantial architectures, we provide empirical evidence that diversity loss is an inherent characteristic of the standard SFT pipeline rather than an artifact of specific smaller-scale architectures.

\begin{figure}[t]
  \vskip 0.2in
  \begin{center}
    \centerline{\includegraphics[width=\columnwidth]{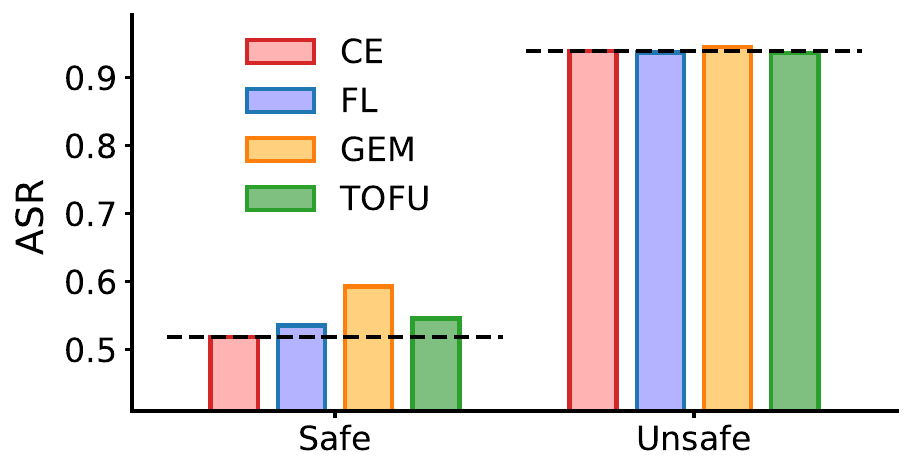}}
    \caption{
    Evaluation results of fine-tuning methods for the Malicious Instruct benchmark averaged across the safe (security aligned) and unsafe (base) models. The values represent the mean ASR scores (0-1) across all tasks. Lower scores indicate safer models. Error bars (standard deviation) are omitted to maintain
visual clarity across multiple model comparisons. The dotted line serves as a reference point for CE performance. 
    }
    \label{fig:safety}
  \end{center}
\end{figure}

We initially attributed the decline to two primary factors: the erosion of pretrained information and the ignorance of low-frequency patterns in the fine-tuning data. This hypothesis is supported by the performance of GEM, designed to mitigate forgetting, and Focal Loss, designed to prioritize underrepresented samples, both of which yield measurable improvements in diversity compared to CE. Through careful theoretical analysis, we created a novel training objective, TOFU, which tackles both challenges simultaneously. 

TOFU achieves the highest diversity while maintaining highly competitive quality across our creative writing and instruction following benchmarks. Furthermore, we found that, in mathematical Chain-of-Thought reasoning, TOFU encourages a higher exploration mode, thereby increasing the probability of capturing a correct solution. This amplified response breadth might provide a superior foundation for the Reinforcement Learning alignment phase. Additionally, we showed that expanded diversity does not come at the cost of factual integrity and safety alignment. Altogether, these results position TOFU as a robust framework for improving model expressivity, enabling greater functional utility across a wide range of downstream applications.

While TOFU demonstrated stable performance across various model sizes and SFT datasets, there remains potential for further validation. Given greater computational resources, future experiments could evaluate these objectives on larger-scale architectures and higher-quality closed-source instruction tuning datasets. Finally, while this study focuses on the supervised phase, the exact influence of the diversity-enhancing objectives on subsequent Reinforcement Learning stages remains an open question for future investigation. We hope that our results encourage further research into specialized training objectives that leverage theoretical insights to refine model performance. Additionally, we look forward to seeing the community apply more extensive resources to test and build upon these findings.

\paragraph{Reproducibility Statement}
The code is available at \url{https://github.com/rsklypa/TOFU}, the datasets are available at \url{https://huggingface.co/TOFU-SFT}.

\section*{Impact Statement}

This work aims to advance the field of Machine Learning by addressing generative diversity and knowledge retention in fine-tuned Large Language Models. By mitigating the forgetting of pretrained information and the underrepresentation of rare patterns, our research supports the development of models that produce more varied and information-rich outputs. There are many potential societal consequences of improving generative variety, none of which we feel must be specifically highlighted here beyond the ethical considerations standard to the advancement of language modeling.

\section*{Acknowledgements}
The authors acknowledge the National Academic Infrastructure for Supercomputing in Sweden (NAISS) for granting this project access to high-performance clusters. The computations and data handling were enabled by Alvis, provided by Chalmers e-Commons at Chalmers, and Berzelius, provided by the Knut and Alice Wallenberg Foundation at the National Supercomputer Centre by the National Academic Infrastructure for Supercomputing in Sweden (NAISS). During this project, O.C. was supported by the SciLifeLab \& Wallenberg Data Driven Life Science Program (a DDLS Academic PhD grant to Eric Libby and Laura Michelle Carroll).

\bibliography{references}
\bibliographystyle{icml2026}

\newpage
\onecolumn
\section*{Appendix}
\appendix

\counterwithin{figure}{section}
\counterwithin{table}{section}
\counterwithin{equation}{section}
\counterwithin{algorithm}{section}

\renewcommand{\thefigure}{\Alph{section}.\arabic{figure}}
\renewcommand{\thetable}{\Alph{section}.\arabic{table}}
\renewcommand{\theequation}{\Alph{section}.\arabic{equation}}
\renewcommand{\thealgorithm}{\Alph{section}.\arabic{algorithm}}

\section{Omitted proofs}
\label{appendix:proofs}
\begin{proposition}[Notation and Logit Gradients]
\label{pr:logit}
We introduce the notation used throughout the following proofs and derivations. All expressions below are written for a single token, with indices ranging over the vocabulary (i.e., number of classes). Consider a temperature-scaled distribution $p^\beta$ from Definition (\ref{def:gem}), a Kronecker's delta matrix $\delta_{ij}$, and $z$ as model's output logits. Let us denote the following log probabilities
    \begin{gather}
    l_i = \log p_i = z_i - \log \sum_k e^{z_k} \\
    l^\beta_i = \log p^\beta_i = \frac{z_i}{\beta} - \log \sum_k e^{\frac{z_k}{\beta}} 
    \end{gather}
Then their derivatives with respect to the logits are given by 
    \begin{gather}
    \frac{\partial l_i}{\partial z_j} =  \delta_{ij} - p_j = \delta_{ij} - e^{l_j} \\
    \frac{\partial l^\beta_i}{\partial z_j} = \frac{1}{\beta}(\delta_{ij} - p^\beta_j)
\end{gather}
\begin{proof}
The computation is straightforward, partial derivatives of log probabilities are computed as
\begin{gather}
    \frac{\partial l_i}{\partial z_j} = \delta_{ij} - \frac{e^{z_j}}{\sum_k e^{z_k}} = \delta_{ij} - e^{l_j} = \delta_{ij} - p_j, \\
    \frac{\partial l^\beta_i}{\partial z_j} = \frac{1}{\beta}\delta_{ij} - \frac{1}{\beta} \frac{e^{\frac{z_j}{\beta}}}{\sum_ke^{\frac{z_k}{\beta}}}= \frac{1}{\beta}(\delta_{ij} - p^\beta_j).
\end{gather}
\end{proof}
\end{proposition}

While the final loss gradient with respect to the model parameters is independent of whether intermediate derivatives are computed via $l$ or $p$, practical implementations that omit the explicit calculation of $p$ make derivatives with respect to $l$ more informative. We illustrate this on the Shannon entropy example:
\begin{equation}
    \mathcal{H} \doteq -p_i \log p_i = -p_{i}l_{i}.
\end{equation}

Its gradients with respect to $p$ are
\begin{equation}
\label{eq:hdiv}
    \frac{\partial\mathcal{H}}{\partial p_i} =  \frac{\partial}{\partial p_i}(-p_i \log p_i )=-\log p_i -1.
\end{equation} 
Based on the eq. (\ref{eq:hdiv}) in low probability setting when $p\rightarrow 0$ the Shannon entropy gradient diverges $\frac{\partial\mathcal{H}}{\partial p_i} \rightarrow \infty$. We now examine the gradients with respect to log probabilities, as these are the values calculated in practical implementations:
\begin{gather}
    \label{eq:nohdiv}
    \frac{\partial\mathcal{H}}{\partial l_i} = \frac{\partial}{\partial l_i} (-p_{i}l_{i}) = -l_i \frac{\partial p_i}{\partial l_i}- p_i=- l_i p_i - p_i = - p_i \log p_i -p_i.
\end{gather}
From the eq. (\ref{eq:nohdiv}) it is clear that there is no explosion as when $p \rightarrow 0$, the gradients are $\frac{\partial\mathcal{H}}{\partial l_i} \rightarrow 0$ due to the following limit:
\begin{equation}
   \lim_{p\rightarrow0} (l p - p) = \lim_{p\rightarrow0}( p \log p -p) = 0.
\end{equation} 
The gradients remain robust throughout the rest of backpropagation chain:
\begin{equation}
    \frac{\partial\mathcal{H}}{\partial z_j} = \sum_i\frac{\partial\mathcal{H}}{\partial l_i} \frac{\partial l_i}{\partial z_j} =  \sum_i- (l_i p_i + p_i)(\delta_{ij} - p_j) = -l_j p_j +\sum_i l_i p_i p_j = -p_j \big(l_j - \sum_i l_i p_i \big)
\end{equation}
This example demonstrates how using $l$ rather than $p$ helps avoid false assumptions regarding the method's behavior.

Proposition~\ref{pr:gem} derives the formulation of the GEM loss used in Definition~\ref{def:gem}, translating the original objective into our notation and providing a simplified expression that is more comprehensible on sight.

\begin{proposition}[Reformulation of the original GEM from \cite{li_preserving_2024} with our notation]
\label{pr:gem}
Consider a model distribution $p_\theta$, data distribution $q$ and a temperature-scaled distribution $p_{\theta}^{\beta}=\text{softmax}(\beta^{-1} \log p_{\theta})$.
Consider $y^{\text{real}}$ is the supervised label in the dataset and $y^{\text{gene}}$ is the model-generated output.
The original GEM formulation is provided in the following equation:
\begin{gather}
\label{eq:orig gem}
\mathcal{L}_{\mathrm{GEM}}(\theta) = \sum_{y^{\text{real}} \sim q}\sum_{y^{\text{gene}} \sim p_{\theta}}\red{p_{\theta}^{\beta}(y^{\text{gene}})}\left[\log p_{\theta}(y^{\text{gene}}) - \log p_{\theta}(y^{\text{real}})\right],
\end{gather}
where $\beta \in (0,1)$ is a temperature parameter and $\red{p^{\beta}(y^{\text{gene}})}$ is detached from gradient computational graph. 
The loss (\ref{eq:orig gem}) in our notation can be reformulated as
\begin{equation}
    \mathcal{L}_{\mathrm{GEM}}(\theta) = -  \mathbb{E}_{q} [\log p_\theta] + \mathbb{E}_{\red{p^\beta_\theta}} [\log p_\theta].
\end{equation}

\begin{proof}
Let us rewrite the eq. (\ref{eq:orig gem}) with respect to the detached term:
\begin{gather}
\mathcal{L}_{\mathrm{GEM}}(\theta) = \mathbb{E}_{y^{\text{real}} \sim q}\mathbb{E}_{y^{\text{gene}} \sim \red{p_{\theta}^\beta}}\left[\log p_{\theta}(y^{\text{gene}}) - \log p_{\theta}(y^{\text{real}})\right] =
\\
= \mathbb{E}_{y^{\text{real}} \sim q}\mathbb{E}_{y^{\text{gene}} \sim \red{p_{\theta}^\beta}}\left[\log p_{\theta}(y^{\text{gene}})\right] - \mathbb{E}_{y^{\text{real}} \sim q}\mathbb{E}_{y^{\text{gene}} \sim \red{p_{\theta}^\beta}}\left[\log p_{\theta}(y^{\text{real}})\right] =
\\
=\mathbb{E}_{y^{\text{gene}} \sim \red{p_{\theta}^\beta}}\left[\log p_{\theta}(y^{\text{gene}})\right] - \mathbb{E}_{y^{\text{real}} \sim q}\left[\log p_{\theta}(y^{\text{real}})\right].
\end{gather}
Now, by disregarding notational $y^{\text{gene}}$ and $y^{\text{real}}$, it is clear that GEM exactly matches the following:
\begin{equation}
    \mathcal{L}_{\mathrm{GEM}}(\theta) = -  \mathbb{E}_{q} [\log p_\theta] + \mathbb{E}_{\red{p^\beta_\theta}} [\log p_\theta].
\end{equation}
\end{proof}
\end{proposition}

\begin{theorem}[Theorem \ref{th:gem}]
    \label{th_app:gem}
    Training with $\mathcal{L}_{\mathrm{GEM}}$ is equivalent to training with temperature-scaled Cross-Entropy loss, as 
    \begin{equation}
        \nabla_\theta \mathcal{L}_{\mathrm{GEM}}(\theta) = \nabla_\theta \mathcal{L}^\beta_{\mathrm{CE}}(\theta),
    \end{equation}
    where $\mathcal{L}^\beta_{\mathrm{CE}}(\theta) \doteq -\beta \mathbb{E}_{q} \log p^\beta_\theta$.
\begin{proof}
    The Cross-Entropy loss is defined as $\mathcal{L}_{\mathrm{CE}}=-q_i\log p_i = -q_i l_i$.
    Using the chain rule and Proposition \ref{pr:logit}, Cross-Entropy gradients with respect to log probabilities and logits are
    \begin{gather}
    \frac{\partial\mathcal{L}_{\mathrm{CE}}}{\partial l_i} = - q_i, \\
    \frac{\partial\mathcal{L}_{\mathrm{CE}}}{\partial z_j}= \sum_i\frac{\partial\mathcal{L}_{\mathrm{CE}}}{\partial l_i} \frac{\partial l_i}{\partial z_j}=-\sum_iq_i (\delta_{ij} -p_j) =p_j - q_j.
    \end{gather}
    Similarly, gradients for $\mathcal{L}_{\mathrm{GEM}} = -q_i \log p_i + p_i^\beta \log p_i$ are
    \begin{gather}
        \frac{\partial\mathcal{L}_{\mathrm{GEM}}}{\partial z_j} = \sum_i \frac{\partial\mathcal{L}_{\mathrm{GEM}}}{\partial l_i} \frac{\partial l_i}{\partial z_j} = \sum_i (p_i^\beta - q_i)(\delta_{ij} - p_j) = p_j^\beta - q_j.
    \end{gather}
    Let us calculate the gradients for scaled Cross-Entropy $\mathcal{L}^\beta_{\mathrm{CE}}=-\beta q_i\log p^\beta_i = -\beta  q_i l_i^\beta$.
    \begin{gather}
    \frac{\partial\mathcal{L}^\beta_{\mathrm{CE}}}{\partial z_j}= \sum_i\frac{\partial\mathcal{L}^\beta_{\mathrm{CE}}}{\partial l^\beta_i} \frac{\partial l^\beta_i}{\partial z_j}=-\sum_i \beta q_i \cdot \frac{1}{\beta} (\delta_{ij} -p^\beta_j) =p^\beta_j - q_j.
    \end{gather}
    Now, as $\nabla_\theta z$ does not depend on the loss function, we have
    \begin{equation}
        \nabla_\theta \mathcal{L}_{\mathrm{GEM}}(\theta) = \nabla_\theta \mathcal{L}^\beta_{\mathrm{CE}}(\theta).
    \end{equation}
\end{proof}
\end{theorem}

\begin{proposition}[Proposition \ref{pr:focal}]
    \label{pr_app:focal}
    Assume that the target distribution $q$ is one-hot. Then, for Focal Loss $\mathcal{L}_{\mathrm{FL}}$ and Cross-Entropy $\mathcal{L}_{\mathrm{CE}}$, the gradients satisfy
    \begin{equation}
        \nabla_\theta \mathcal{L}_{\mathrm{FL}}(\theta) = g(\hat{p}_\theta,\gamma) \nabla_\theta \mathcal{L}_{\mathrm{CE}}(\theta),
    \end{equation}
    where $g(p, \gamma) = (1-p)^\gamma - \gamma p (1-p)^{\gamma-1} \log p$. Here and further $\hat{p}$ denotes the predicted probability assigned to the ground-truth token.
\begin{proof}
    First let's consider $q$ to be arbitrary. Then by the chain rule and Proposition \ref{pr:logit} the gradients of Focal Loss are
    \begin{gather}
        \frac{\partial\mathcal{L}_{\mathrm{FL}}}{\partial l_i} = -(1-p_i)^\gamma q_i + \gamma (1-p_i)^{\gamma - 1} q_i l_i p_i, \\ 
        \frac{\partial\mathcal{L}_{\mathrm{FL}}}{\partial z_j} =  \sum_i  \frac{\partial\mathcal{L}_{\mathrm{FL}}}{\partial l_i}\frac{\partial l_i}{\partial z_j}=  -(1-p_j)^\gamma q_j + \gamma (1-p_j)^{\gamma - 1} q_j l_j p_j + \\
    +\sum_i ((1-p_i)^\gamma q_i p_j - \gamma (1-p_i)^{\gamma - 1} q_i l_i p_i p_j) = \\
    = -q_j ((1-p_j)^\gamma - \gamma (1-p_j)^{\gamma - 1} l_j p_j) + p_j \sum_i ((1-p_i)^\gamma q_i - \gamma (1-p_i)^{\gamma - 1} q_i l_i p_i).
    \end{gather}
    In this case, the resulting gradients are not proportional to the gradients of the Cross-Entropy loss. However, if $q$ is one-hot, meaning $q_k = 1$ and $q_{i \neq k} = 0$ for some $k$, then 
    \begin{gather}
        \frac{\partial\mathcal{L}_{\mathrm{FL}}}{\partial z_j} = \underbrace{((1-p_k)^\gamma - \gamma (1-p_k)^{\gamma - 1} l_k p_k)}_{g(p_k, \gamma)} (p_j - q_j),
    \end{gather}
    and therefore, as $\nabla_\theta z$ does not depend on the loss function and denoting $(1-p)^\gamma - \gamma p (1-p)^{\gamma-1} \log p$ as $g(p, \gamma)$, we have
    \begin{gather}
        \nabla_\theta \mathcal{L}_{\mathrm{FL}}(\theta) = g(\hat{p}_\theta,\gamma) \nabla_\theta \mathcal{L}_{\mathrm{CE}}(\theta).
    \end{gather}
\end{proof}
\end{proposition}

\begin{corollary}[Corollary \ref{cor:focalb}]
    \label{cor_app:focalb}
    If the target distribution $q$ is one-hot
    \begin{equation}
        -\beta \nabla_\theta \mathbb{E}_{q} \left [ (1-p_\theta^\beta)^\gamma \log p_\theta^\beta \right ] = g(\hat{p}_\theta^\beta, \gamma) \nabla_\theta \mathcal{L}^\beta_{\mathrm{CE}}(\theta)
    \end{equation}
\begin{proof}
    By the chain rule, Proposition \ref{pr:logit}, and analogously to the proof of Proposition \ref{pr:focal}, if the target distribution $q$ is one-hot
    \begin{gather}
        \frac{\partial}{\partial z_j} \sum_i (1-p_i^\beta)^\gamma q_i \log p_i^\beta = \frac{1}{\beta}\underbrace{((1-p^\beta_k)^\gamma - \gamma (1-p^\beta_k)^{\gamma - 1} l^\beta_k p^\beta_k)}_{g(\hat{p}_k^\beta, \gamma)} (q_j - p^\beta_j).
    \end{gather}
    Therefore, as $\nabla_\theta z$ does not depend on the loss function
    \begin{equation}
        -\beta \nabla_\theta \mathbb{E}_{q} \left [ (1-p_\theta^\beta)^\gamma \log p_\theta^\beta \right ] = g(\hat{p}_\theta^\beta, \gamma) \nabla_\theta \mathcal{L}^\beta_{\mathrm{CE}}(\theta).
    \end{equation}
\end{proof}
\end{corollary}

\begin{corollary}[Corollary \ref{cor:tony}]
    \label{cor_app:tony}
    If the target distribution $q$ is one-hot, TOFU gradients are proportional to the ones of the temperature-scaled CE: 
    \begin{gather}
        \nabla_\theta \mathcal{L}_{\mathrm{TOFU}}(\theta) = g(\hat{p}_\theta,\gamma)\nabla_\theta \mathcal{L}^\beta_{\mathrm{CE}}(\theta)
    \end{gather}
\begin{proof}
    As $g(\hat{p}_\theta,\gamma)$ is detached from gradients computation
    \begin{gather}
        \nabla_\theta \mathcal{L}_{\mathrm{TOFU}}(\theta) = \nabla_\theta \red{g(\hat{p}_\theta,\gamma)} \mathcal{L}^\beta_{\mathrm{CE}}(\theta) = g(\hat{p}_\theta,\gamma)\nabla_\theta \mathcal{L}^\beta_{\mathrm{CE}}(\theta)
    \end{gather}
\end{proof}
\end{corollary}

\newpage
\section{Experimental details}
\label{appendix:details}

In this section we provide comprehensive descriptions for the models, benchmarks and datasets used in our study.

\subsection{Models}

\paragraph{OLMo-2-13B}
OLMo-2-1124-13B \cite{groeneveld_olmo_2024} from the Allen Institute for AI, trained on the Dolma \cite{soldaini_dolma_2024} dataset for improved performance on tasks such as text generation and instruction following. This model is released under Apache 2.0 license.

\paragraph{Mistral-12B}
Mistral NeMo \cite{mistral_ai_mistral_2024} is trained jointly by Mistral AI and NVIDIA. It is designed for diverse tasks including text generation and instruction following. This model is released under Apache 2.0 license.

\paragraph{Pythia-12B}
Pythia 12B \cite{biderman_pythia_2023} is trained on the Pile \cite{gao_pile_2020} as a scientific tool for studying model functionality and interpretability rather than for deployment or human-facing interactions. This model is released under Apache 2.0 license.

\paragraph{Llama-3.1-8B}
Llama-3.1 \cite{touvron_llama_2023} is released by Meta AI as an extension of the Llama-3 series. The model serves as a strong foundation for downstream fine-tuning and alignment methods, making it widely adopted in both research and applied settings. This model contains custom Llama-3.1 license \footnote{\url{https://github.com/meta-llama/llama-models/blob/main/models/llama3_1/LICENSE}}.

\paragraph{Qwen-3-8B}
Qwen3 \cite{yang_qwen3_2025} is a family of large language models developed by Alibaba Cloud, designed to support general-purpose language understanding, reasoning, and instruction-following tasks. The model is trained on a diverse mixture of web, code, and domain-specific data. This model is released under Apache 2.0 license.

\paragraph{Phi-4-14B}
Phi-4 \cite{abdin_phi-4_2024} is a model family developed by Microsoft, focusing on efficiency and strong reasoning capabilities under limited parameter budgets. It is trained using a carefully curated dataset that emphasizes high-quality, synthetic, and textbook-style data. This model is released under MIT license.

\paragraph{Qwen2.5-Math} Qwen2.5-Math \cite{yang_qwen25-math_2024} is a branch of math-specific large language models from Qwen family. The models in this series possess advanced mathematical reasoning capabilities, including Chain-of-Thought (CoT). This model is released under Apache 2.0 license.

\paragraph{DeepSeek-Math-7B} DeepSeek-Math \cite{shao_deepseekmath_2024} is a collection of models pre-trained on math-related tokens sourced from Common Crawl, together with natural language and code data for 500B tokens. This model is released under MIT license.

\subsection{SFT Datasets}
\paragraph{Alpaca} Alpaca\footnote{\url{https://huggingface.co/datasets/tatsu-lab/alpaca}} \cite{taori_stanford_2023} is a widely used instruction-following dataset consisting of approximately 52K instruction–response pairs generated using a self-instruct framework. The dataset covers a broad range of tasks, including question answering, summarization, reasoning, and creative writing. To preprocess the Alpaca dataset, we filter and format each example into prompt–completion pairs. We use explicit delimiters for the instruction, input, and response to provide structural context for the sequence. The model is trained to generate the response following an opening delimiter and is explicitly required to produce a matching delimiter to signal completion. Alpaca is available under CC-BY-NC-4.0 license.

\paragraph{UltraFeedback} UltraFeedback \cite{cui_ultrafeedback_2023} is a preference-annotated dataset hosted on Hugging Face by openbmb, derived from the UltraFeedback corpus and adapted for supervised fine-tuning and instruction tuning. The dataset comprises 64K samples, each containing an AI-generated judgment that identifies a preferred response. Our pipeline filters the UltraFeedback dataset for top-quality examples, retaining only the highest-scoring completion per prompt, provided that it meets a minimum score of 7, resulting in 57,400 samples. We format these pairs using fixed templates the same as for the Alpaca dataset. UltraFeedback is available under MIT license.

\paragraph{NuminaMath CoT}  NuminaMath \cite{numina} is a dataset of approximately 860,000 mathematical problems, designed to examine the model's reasoning in Chain of Thought (CoT) manner. The dataset covers a wide range of tasks, from Chinese high school math exercises to US and international mathematics olympiad competition problems. We filtered 100,000 examples including all sources, ensuring that the solution is given in a bounding box format. This dataset is available under Apache 2.0 license.

\subsection{SFT details}
\label{SFT details}
Given the limitation of our computational resources, we performed 4-bit NormalFloat quantization of selected models and utilized the Quantized Low Rank Adaptation \cite{dettmers_qlora_2023} technique to optimize our workflow. On top of that, we used gradient accumulation to increase the total batch size. 

All models were trained for a single epoch using a linear learning rate schedule with a peak of $2\times10^{-4}$ and 50 warmup steps. We employed a batch size of 2 with 4 gradient accumulation steps and a weight decay of 0.01. For the LoRA adapter, we set $r=16$ and $\alpha=16$.

\subsection{Evaluation Datasets}

\textbf{Short Stories} is a continuation benchmark where a model is given a story beginning and must generate a coherent conclusion. We constructed this dataset by randomly selecting 100 different stories from the ROCStories corpus \cite{mostafazadeh_corpus_2016}, each containing exactly five sentences. For our evaluation, the first four sentences serve as the beginning of the story, providing sufficient context for a logical continuation.
    
\textbf{Small Prompts} benchmark is a collection of short questions from the \textit{helpful\_base} subset of the AlpacaFarm \cite{dubois_alpacafarm_2023} Hugging Face repository, comprising 129 prompts. We take only this portion of the original dataset to ensure that the evaluation remains focused on standard natural language (in opposition to code or ASCII symbol drawings), as the Self-BLEU metric does not function reliably outside of this domain.

\textbf{NoveltyBench} \cite{zhang_noveltybench_2025} is a benchmark designed to evaluate language models’ ability to generate multiple distinct and high-quality outputs for the same prompt, removing the traditional focus from a single best response. For the evaluation, we selected its \textit{NB-curated} subset, which contains 100 manually curated prompts. We utilized the original NoveltyBench framework and source code, including the default parameters for their proprietary quality and diversity metrics, Utility-k and Distinct-k. The code is available under MIT license.

\textbf{Massive Multitask Language Understanding} (MMLU) \cite{hendrycks_measuring_2020} is a benchmark designed to evaluate the knowledge and reasoning capabilities of language models across multiple subject areas, spanning STEM disciplines, humanities, and social sciences. The dataset includes questions of varying difficulty levels, ranging from elementary concepts to advanced professional knowledge. For the benchmarking, we used its test subset, comprising 14042 questions. This dataset is available under MIT license.

\textbf{ARC-Challenge} (ARC) \cite{clark_think_2018} is a benchmark dataset of multiple-choice science questions curated to evaluate advanced reasoning and scientific understanding. The questions are sourced from standardized science examinations for grades 3 through 9 and are intentionally selected to be challenging for both humans and AI systems. For the benchmarking, we used its test subset, comprising 1172 questions. This dataset is available under CC-BY-SA-4.0 license.

\textbf{MATH500} is an open-source test subsample of the original MATH dataset \cite{hendrycks_measuring_2021}. It comprises 500 problems alongside their solutions in various subjects, such as algebra, geometry, calculus, and probability. The code accompanying this dataset is available under MIT license.

\textbf{MinervaMath} is a publicly available subset of Minerva corpus \cite{lewkowycz_solving_2022}, consisting of 272 mathematical problems related to natural sciences. This dataset is available under MIT license.

\textbf{GSM8K} is a collection of 8,800 high quality linguistically diverse grade school math word problems \cite{cobbe_training_2021}. For the evaluation we used a main test set of 1,320 tasks. This dataset is available under MIT license.

\textbf{Malicious Instruct} is a human crafted dataset \cite{huang_catastrophic_2023}, consisting of 100 prompts. To construct the datasets the authors selected ten categories and asked ChatGPT to provide 20 responses for each of the categories. They manually reviewed the generated responses and selected 100 responses such that they are aligned with the topic and diverse at the same time.

\textbf{Harm Bench} is a standardized evaluation framework for automated red teaming \cite{mazeika_harmbench_2024}, consisting of 400 malicious prompts in different categories: copyright, contextual, and standard. For our main experiments, we randomly selected 100 prompts from the standard subset. This dataset is available under MIT license.
\subsection{Inference details}

The model’s generation parameters were selected based on the specific requirements of each evaluation task. Following standard empirical practices in the field, we employed a stochastic sampling strategy for the creative writing \& instruction following, mathematical reasoning and safety benchmarks. Specifically, we used nucleus sampling \cite{holtzman_curious_2020} with a cumulative probability threshold of $p=0.9$ and different unit temperatures, depending on the task. 

For instruction following and safety we employed $T=1.0$ for all the training objectives, while in mathematical reasoning experiments each loss utilizes a specific temperature. While we maintained $T=1.0$ for CE, GEM and TOFU exhibited high performance variance at this default, leading to non-robust evaluations where per-run fluctuations determined the top-performing method. This instability is inherent to the reasoning task, where a single incorrect token can derail the entire Chain-of-Thought, and minor deviations from the required format result in the response being classified as incorrect. Therefore, we empirically derived $T=0.3$ such that the variance of the coverage across the inference runs for each objective stays in range of $1\%$. For creative writing, instruction following, and safety benchmarks we generated 10 responses per prompt, and for mathematical reasoning we employed 16 decoding runs. 
In contrast, for the ARC and MMLU, we used deterministic greedy decoding to ensure objective and reproducible outputs, generating a single completion per prompt. 

Due to computational constraints, the maximum response length was restricted across all the experiments. We allocated a limit of 64 tokens for creative benchmarks (Short Stories, Small Prompts, Novelty Bench) and 8 tokens for multiple-choice ones (ARC, MMLU). For Malicious Instruct and HarmBench datasets we used limit of 128 tokens. For experiments with mathematical reasoning in a Chain-of-Thought (CoT) manner, we set the maximal generation length of 4096 tokens. 

For the inference of SFT models, we adhere to the standard instruction template, incorporating an additional prompt if necessary, depending on the benchmark. Regarding the base models, we observed that they successfully continue the narrative on the SS benchmark when provided with substantial initial context, even without explicit instructions. This enables a direct comparison with their fine-tuned counterparts. For the SP benchmark, we follow the protocol established by \citet{omahony_attributing_2024}. All prompt templates are provided in the accompanying code repository\footnote{\url{https://github.com/rsklypa/TOFU}}.

\subsection{LLM Judge}
\label{app:judge}

To assess the quality of the responses, we employ large language model as a judge to score and compare generated responses on Short Stories and Small Prompts datasets.  Specifically, we use \textbf{Llama-3.1-70B-Instruct} \cite{grattafiori_llama_2024}, an instruction-aligned model. To reduce memory footprint, we quantized the Judge with 4-bit NormalFloat. 
We prompted the Judge with comprehensive instructions to provide a score from 0 to 5, where 0 corresponds to an incoherent, off-topic, or nonsensical response, and 5 corresponds to a seamless, natural, and stylistically consistent one. Given the restriction in tokens that we applied in the inference stage, we explicitly state in the Judge instruction not to penalize the response if it ends abruptly due to the aforementioned limit. However, it is required to penalize a logically or stylistically flawed ending. To ensure a robust evaluation, we used greedy decoding and strict response template.

\section{Ablations}
\label{appendix:ablations}
We find best parameters for Focal Loss and TOFU objectives by evaluating quality and diversity on NoveltyBench and accuracy on ARC. For these experiments we employ Mistral-12B model. We tested the following ranges of parameters: $\gamma \in [2,5]$ and $\beta \in [0.6,0.9]$. According to the results, gathered in Table \ref{tab:ablation}, the best hyperparameters are $\gamma=3$ for Focal Loss and $\gamma=3, \beta=0.8$ for TOFU. While we selected the optimal values primarily based on ARC accuracy, they also coincide with the top-performing configurations for NoveltyBench Utility. In contrast with the rest of the experiments on NoveltyBench, here we generated only 5 responses per prompt.

\begin{table}[h!]
\caption{ARC Accuracy, NoveltyBench Distinct (1-10) and Utility (1-10) results for the Mistral-12B model fine-tuned with Focal Loss and TOFU with different $(\gamma, \beta)$ configurations.}
\label{tab:ablation}
\begin{center}
\begin{small}
\begin{sc}
\setlength{\tabcolsep}{4pt}
\begin{tabular}{ccccccccccccccccc}
\toprule
& \multicolumn{4}{c}{Focal Loss} & \multicolumn{12}{c}{TOFU} \\
\midrule
$\boldsymbol{\beta}$ & - & - & - & - & 0.6 & 0.6 & 0.6 & 0.7 & 0.7 & 0.7 & 0.8 & \textbf{0.8} & 0.8 & 0.9 & 0.9 & 0.9 \\
$\boldsymbol{\gamma}$ & 2 & \textbf{3} & 4 & 5 & 2 & 3 & 5 & 2 & 3 & 5 & 2 & \textbf{3} & 5 & 2 & 3 & 5 \\
\midrule
Accuracy $\uparrow$ & 72.9 & \textbf{75.0} & 74.3 & 74.2 & 72.0 & 74.4 & 75.4 & 73.1 & 73.8 & 74.6 & 75.3 & \textbf{75.6} & 74.5 & 71.6 & 74.7 & 74.6 \\
Distinct $\uparrow$ & 3.97 & \textbf{4.23} & 4.44 & 4.37 & 4.64 & 4.69 & 4.76 & 4.46 & 4.40 & 4.69 & 4.25 & \textbf{4.41} & 4.54 & 4.18 & 4.36 & 4.45 \\
Utility $\uparrow$ & 4.61 & \textbf{4.90} & 4.81 & 4.29 & 4.30 & 4.13 & 3.97 & 4.53 & 4.37 & 4.17 & 4.72 & \textbf{4.71} & 4.22 & 4.76 & 4.65 & 4.16 \\
\bottomrule
\end{tabular}
\end{sc}
\end{small}
\end{center}
\vskip -0.1in
\end{table}

\section{Additional results}
\label{appendix:additional}

To verify the robustness of our results across different SFT datasets, we replicated the experiments conducted on Alpaca using UltraFeedback. While the lower overall quality of UltraFeedback negatively affects performance, the trends observed in the Alpaca experiments remain preserved. The corresponding results for Short Stories and Small Prompts are provided in Table~\ref{tab:ufsssp}.
For diversity metrics, we report mean values and standard deviations calculated across the prompts. For quality, we first compute the average score per prompt, then report the global mean and the standard deviation of those per-prompt averages.

\begin{table*}[!ht]
\caption{Performance of models across UltraFeedback SFT objectives on Short Stories and Small Prompts. Diversity (D) is measured via Self-BLEU (0–100), where lower scores are better. Quality (Q) is measured via LLM Judge score (0–5), where higher scores are better.}
\label{tab:ufsssp}
\begin{center}
\begin{small}
\begin{sc}
\setlength{\tabcolsep}{1.5pt}
\begin{tabular}{c l c c c c c c c c c c c c}
\toprule
Bench & Method & \multicolumn{2}{c}{Mistral-12B} & \multicolumn{2}{c}{OLMo-2-13B} & \multicolumn{2}{c}{Pythia-12B} & \multicolumn{2}{c}{Llama-3.1-8B} & \multicolumn{2}{c}{Qwen-3-8B} & \multicolumn{2}{c}{Phi-4-14B} \\

& & D$\downarrow$ & Q$\uparrow$ & D$\downarrow$ & Q$\uparrow$ & D$\downarrow$ & Q$\uparrow$ & D$\downarrow$ & Q$\uparrow$ & D$\downarrow$ & Q$\uparrow$ & D$\downarrow$ & Q$\uparrow$ \\

\midrule

\multirow{6}{*}[-2.5ex]{SS} & Base & 11.4$_{\pm 4.7}$ & 3.9$_{\pm 0.5}$ & 12.3$_{\pm 4.6}$ & 3.7$_{\pm 0.7}$ & 9.3$_{\pm 3.1}$ & 2.9$_{\pm 0.6}$ & 11.5$_{\pm 6.1}$ & 3.5$_{\pm 0.7}$ & 24.9$_{\pm 8.9}$ & 2.9$_{\pm 1.0}$ & 13.8$_{\pm 6.8}$ & 3.4$_{\pm 0.9}$ \\[1ex]

& CE & 26.9$_{\pm 9.8}$ & 4.4$_{\pm 0.4}$ & 24.0$_{\pm 8.7}$ & 4.4$_{\pm 0.5}$ & 19.7$_{\pm 7.4}$ & 3.1$_{\pm 0.7}$ & 23.9$_{\pm 8.8}$ & 4.2$_{\pm 0.5}$ & 25.6$_{\pm 12.0}$ & 4.2$_{\pm 0.5}$ & 24.4$_{\pm 7.9}$ & 4.5$_{\pm 0.3}$ \\

\cmidrule(){2-14}

& \textcolor{grey}{$\lambda$-PR} & \textcolor{grey}{3.3$_{\pm 0.5}$} & \textcolor{grey}{1.8$_{\pm 0.5}$} & \textcolor{grey}{3.2$_{\pm 0.4}$} & \textcolor{grey}{2.1$_{\pm 0.5}$} & \textcolor{grey}{3.6$_{\pm 0.7}$} & \textcolor{grey}{1.2$_{\pm 0.5}$} & \textcolor{grey}{3.3$_{\pm 0.5}$} & \textcolor{grey}{1.9$_{\pm 0.5}$} & \textcolor{grey}{4.0$_{\pm 0.8}$} & \textcolor{grey}{2.6$_{\pm 0.5}$} & \textcolor{grey}{3.4$_{\pm 0.5}$} & \textcolor{grey}{2.2$_{\pm 0.5}$} \\[1ex]

& FL & 15.3$_{\pm 4.7}$ & 4.4$_{\pm 0.3}$ & 14.3$_{\pm 4.6}$ & 4.3$_{\pm 0.4}$ & 14.9$_{\pm 4.7}$ & 2.7$_{\pm 0.6}$ & 14.9$_{\pm 5.4}$ & 4.1$_{\pm 0.4}$ & 15.1$_{\pm 5.7}$ & 4.2$_{\pm 0.4}$ & 15.5$_{\pm 5.4}$ & 4.4$_{\pm 0.3}$ \\[1ex]

& GEM & 13.1$_{\pm 5.3}$ & 4.2$_{\pm 0.4}$ & 12.1$_{\pm 4.3}$ & 4.1$_{\pm 0.5}$ & 9.8$_{\pm 3.6}$ & 2.6$_{\pm 0.7}$ & 12.0$_{\pm 4.8}$ & 4.0$_{\pm 0.4}$ & 13.4$_{\pm 6.5}$ & 4.1$_{\pm 0.4}$ & 13.1$_{\pm 5.8}$ & 4.3$_{\pm 0.4}$ \\[1ex]

& TOFU & \textbf{12.1$_{\pm 4.6}$} & 4.2$_{\pm 0.3}$ & \textbf{11.2$_{\pm 3.5}$} & 4.2$_{\pm 0.4}$ & \textbf{9.5$_{\pm 3.0}$} & 2.9$_{\pm 0.6}$ & \textbf{11.4$_{\pm 3.7}$} & 4.0$_{\pm 0.4}$ & \textbf{12.3$_{\pm 4.7}$} & 4.1$_{\pm 0.4}$ & \textbf{11.8$_{\pm 4.3}$} & 4.3$_{\pm 0.4}$ \\

\midrule

\multirow{6}{*}[-2.5ex]{SP} & Base & 12.7$_{\pm 7.0}$ & 3.8$_{\pm 0.9}$ & 13.9$_{\pm 7.2}$ & 3.8$_{\pm 1.0}$ & 8.4$_{\pm 3.0}$ & 2.6$_{\pm 1.0}$ &12.2$_{\pm 6.0}$ & 3.5$_{\pm 0.9}$ & 31.5$_{\pm 12.5}$ & 3.8$_{\pm 0.9}$ & 17.6$_{\pm 9.3}$ & 3.9$_{\pm 0.9}$ \\[1ex]

& CE & 45.1$_{\pm 10.3}$ & 3.8$_{\pm 0.5}$ & 43.9$_{\pm 9.6}$ & 3.8$_{\pm 0.5}$ & 32.0$_{\pm 9.0}$ & 3.2$_{\pm 0.6}$ & 42.8$_{\pm 10.0}$ & 3.8$_{\pm 0.5}$ & 44.3$_{\pm 9.8}$ & 3.7$_{\pm 0.6}$ & 48.0$_{\pm 9.9}$ & 3.8$_{\pm 0.5}$ \\

\cmidrule(){2-14}

& \textcolor{grey}{$\lambda$-PR} & \textcolor{grey}{2.4$_{\pm 0.5}$} & \textcolor{grey}{1.3$_{\pm 0.6}$} & \textcolor{grey}{2.5$_{\pm 0.6}$} & \textcolor{grey}{1.4$_{\pm 0.6}$} & \textcolor{grey}{2.5$_{\pm 0.4}$} & \textcolor{grey}{1.3$_{\pm 0.5}$} & \textcolor{grey}{2.6$_{\pm 0.5}$} & \textcolor{grey}{1.3$_{\pm 0.5}$} & \textcolor{grey}{3.4$_{\pm 0.8}$} & \textcolor{grey}{1.7$_{\pm 0.6}$} & \textcolor{grey}{2.8$_{\pm 0.7}$} & \textcolor{grey}{1.4$_{\pm 0.5}$} \\[1ex]

& FL & 27.8$_{\pm 7.1}$ & 3.6$_{\pm 0.5}$ & 26.8$_{\pm 7.0}$ & 3.6$_{\pm 0.5}$ & 20.6$_{\pm 5.8}$ & 3.0$_{\pm 0.6}$ & 26.0$_{\pm 7.0}$ & 3.6$_{\pm 0.5}$ & 28.3$_{\pm 7.5}$ & 3.5$_{\pm 0.5}$ & 27.9$_{\pm 7.3}$ & 3.6$_{\pm 0.5}$ \\[1ex]

& GEM & 26.3$_{\pm 8.0}$ & 3.6$_{\pm 0.5}$ & 25.1$_{\pm 7.3}$ & 3.5$_{\pm 0.5}$ & 16.7$_{\pm 5.7}$ & 2.8$_{\pm 0.6}$ & 24.2$_{\pm 6.9}$ & 3.5$_{\pm 0.5}$ & 28.4$_{\pm 8.5}$ & 3.5$_{\pm 0.5}$ & 28.1$_{\pm 8.0}$ & 3.6$_{\pm 0.5}$ \\[1ex]

& TOFU & \textbf{20.9$_{\pm 5.7}$} & 3.5$_{\pm 0.6}$ & \textbf{20.2$_{\pm 5.7}$} & 3.5$_{\pm 0.5}$ & \textbf{14.6$_{\pm 4.9}$} & 2.8$_{\pm 0.6}$ & \textbf{19.8$_{\pm 5.7}$} & 3.5$_{\pm 0.5}$ & \textbf{22.5$_{\pm 6.2}$} & 3.5$_{\pm 0.5}$ & \textbf{21.1$_{\pm 6.2}$} & 3.5$_{\pm 0.5}$ \\

\bottomrule
\end{tabular}
\end{sc}
\end{small}
\end{center}
\vskip -0.1in
\end{table*}

While the primary figures for NoveltyBench are presented in the main text, the corresponding raw values are provided in Table~\ref{tab:nb}. We report the mean values and standard deviations calculated across the prompts.
We note that the Utility metric exhibits a high standard deviation, in some cases exceeding half of the mean value. This significant variance may stem from the diverse difficulty levels of the prompts within the benchmark, or potentially from a lack of robustness in the judge used for evaluation.

\begin{table*}[!ht]
\caption{Performance of models across UltraFeedback and Alpaca SFT objectives on NoveltyBench. (D) Distinct (1–10) measures responses diversity, while (U) Utility (1–10) represents quality. For both metrics, higher values indicate superior performance. }
\label{tab:nb}
\begin{center}
\begin{small}
\begin{sc}
\setlength{\tabcolsep}{3pt}
\begin{tabular}{c l c c c c c c c c c c c c}
\toprule
SFT & Method & \multicolumn{2}{c}{Mistral-12B} & \multicolumn{2}{c}{OLMo-2-13B} & \multicolumn{2}{c}{Pythia-12B} & \multicolumn{2}{c}{Llama-3.1-8B} & \multicolumn{2}{c}{Qwen-3-8B} & \multicolumn{2}{c}{Phi-4-14B} \\

& & D$\uparrow$ & U$\uparrow$ & D$\uparrow$ & U$\uparrow$ & D$\uparrow$ & U$\uparrow$ & D$\uparrow$ & U$\uparrow$ & D$\uparrow$ & U$\uparrow$ & D$\uparrow$ & U$\uparrow$ \\

\midrule

\multirow{5}{*}[-2ex]{\rotatebox{90}{Alpaca}} & CE & 6.6$_{\pm 2.7}$ & 4.3$_{\pm 2.4}$ & 6.5$_{\pm 2.8}$ & 4.2$_{\pm 2.5}$ & 7.7$_{\pm 2.3}$ & 3.2$_{\pm 2.2}$ & 6.5$_{\pm 2.7}$ & 4.2$_{\pm 2.4}$ & 6.5$_{\pm 2.7}$ & 4.2$_{\pm 2.4}$ & 5.8$_{\pm 2.8}$ & 1.7$_{\pm 1.2}$ \\[1ex]

& \textcolor{grey}{$\lambda$-PR} & \textcolor{grey}{9.5$_{\pm 1.2}$} & \textcolor{grey}{1.3$_{\pm 0.5}$} & \textcolor{grey}{9.4$_{\pm 0.4}$} & \textcolor{grey}{1.3$_{\pm 0.6}$} & \textcolor{grey}{9.3$_{\pm 1.4}$} & \textcolor{grey}{1.1$_{\pm 0.4}$} & \textcolor{grey}{9.5$_{\pm 1.4}$} & \textcolor{grey}{1.3$_{\pm 0.6}$} & \textcolor{grey}{9.1$_{\pm 1.7}$} & \textcolor{grey}{1.2$_{\pm 0.5}$} & \textcolor{grey}{9.6$_{\pm 1.2}$} & \textcolor{grey}{1.2$_{\pm 0.5}$} \\[1ex]

& FL & 7.4$_{\pm 2.4}$ & 4.3$_{\pm 2.3}$ & 7.7$_{\pm 2.6}$ & 4.5$_{\pm 2.5}$ & 8.5$_{\pm 2.0}$ & 3.2$_{\pm 2.0}$ & 7.7$_{\pm 2.5}$ & 4.2$_{\pm 2.4}$ & 7.4$_{\pm 2.5}$ & 4.3$_{\pm 2.3}$ & 7.1$_{\pm 2.5}$ & 2.0$_{\pm 1.3}$ \\[1ex]

& GEM & 8.0$_{\pm 2.4}$ & 4.5$_{\pm 2.5}$ & 8.1$_{\pm 2.2}$ & 4.6$_{\pm 2.5}$ & 8.7$_{\pm 1.8}$ & 3.0$_{\pm 2.0}$ & 8.1$_{\pm 2.2}$ & 4.7$_{\pm 2.6}$ & 7.6$_{\pm 2.4}$ & 4.3$_{\pm 2.5}$ & 7.2$_{\pm 2.6}$ & 1.8$_{\pm 1.2}$ \\[1ex]

& TOFU & \textbf{8.2$_{\pm 2.2}$} & 4.4$_{\pm 2.4}$ & \textbf{8.3$_{\pm 2.1}$} & 4.4$_{\pm 2.4}$ & \textbf{8.8$_{\pm 1.9}$} & 3.0$_{\pm 2.0}$ & \textbf{8.3$_{\pm 2.1}$} & 4.3$_{\pm 2.4}$ & \textbf{7.9$_{\pm 2.5}$} & 4.2$_{\pm 2.4}$ &\textbf{7.5$_{\pm 2.5}$} & 2.0$_{\pm 1.3}$ \\

\midrule

\multirow{5}{*}[-2ex]{\rotatebox{90}{UF}} & \textbf{CE} & 7.3$_{\pm 2.4}$ & 3.2$_{\pm 2.2}$ & 7.3$_{\pm 2.5}$ & 3.0$_{\pm 2.1}$ & 8.5$_{\pm 1.8}$ & 1.7$_{\pm 1.1}$ & 7.6$_{\pm 2.3}$ & 2.9$_{\pm 1.9}$ & 7.0$_{\pm 2.5}$ & 2.9$_{\pm 2.0}$ & 6.9$_{\pm 2.4}$ & 2.2$_{\pm 1.3}$ \\[1ex]

& \textcolor{grey}{$\lambda$-PR} & \textcolor{grey}{10.0$_{\pm 0.2}$} & \textcolor{grey}{1.1$_{\pm 0.3}$} & \textcolor{grey}{9.9$_{\pm 0.5}$} & \textcolor{grey}{1.1$_{\pm 0.2}$} & \textcolor{grey}{9.9$_{\pm 0.5}$} & \textcolor{grey}{1.0$_{\pm 0.2}$} & \textcolor{grey}{10.0$_{\pm 0.2}$} & \textcolor{grey}{1.1$_{\pm 0.3}$} & \textcolor{grey}{9.9$_{\pm 0.5}$} & \textcolor{grey}{1.1$_{\pm 0.4}$} & \textcolor{grey}{9.9$_{\pm 0.3}$} & \textcolor{grey}{1.1$_{\pm 0.3}$} \\[1ex]

& FL & 8.5$_{\pm 1.9}$ & 2.9$_{\pm 1.9}$ & 8.8$_{\pm 1.5}$ & 3.0$_{\pm 1.9}$ & 9.3$_{\pm 1.1}$ & 1.6$_{\pm 0.8}$ & 8.9$_{\pm 1.6}$ & 2.8$_{\pm 1.8}$ & 8.2$_{\pm 1.6}$ & 2.9$_{\pm 1.9}$ & 8.4$_{\pm 1.8}$ & 2.3$_{\pm 1.2}$ \\[1ex]

& GEM & 9.0$_{\pm 1.6}$ & 2.8$_{\pm 1.8}$ & 9.1$_{\pm 1.3}$ & 2.8$_{\pm 1.9}$ & 9.5$_{\pm 0.8}$ & 1.4$_{\pm 0.7}$ & 9.1$_{\pm 1.2}$ & 2.6$_{\pm 1.6}$ & 8.8$_{\pm 1.8}$ & 2.8$_{\pm 1.9}$ & 8.6$_{\pm 1.6}$ & 2.0$_{\pm 1.2}$ \\[1ex]

& TOFU & \textbf{9.1$_{\pm 1.2}$} & 2.6$_{\pm 1.7}$ & \textbf{9.1$_{\pm 1.5}$} & 2.9$_{\pm 1.9}$ & \textbf{9.5$_{\pm 0.9}$} & 1.4$_{\pm 0.7}$ & \textbf{9.3$_{\pm 1.3}$} & 2.6$_{\pm 1.8}$ & \textbf{8.8$_{\pm 1.5}$} & 2.8$_{\pm 1.9}$ & \textbf{8.9$_{\pm 1.4}$} & 2.0$_{\pm 1.2}$ \\

\bottomrule
\end{tabular}
\end{sc}
\end{small}
\end{center}
\vskip -0.1in
\end{table*}

To explore the potential longer-horizon collapse that may occur in longer generations, we conducted additional experiments on Short Stories and Small Prompts by employing Mistral-12B, OLMo-2-13B, and Llama-3.1-8B with a limit of 256 and 512 tokens (see Table \ref{tab:256} and Table \ref{tab:512}). Our results are in line with Table \ref{tab:alpacasssp} showing that models fine-tuned with TOFU consistently outperform other approaches across all tested lengths.

We additionally performed SFT on a non-quantized Llama-3.1-8B and evaluated the model's performance on Short Stories and Small Prompts to ensure that employed quantization does not affect the consistency of our results. Then we compared a quantized model against non-quantized across different SFT objectives and gathered the results in Table \ref{tab:quantvsno}. Ultimately, we observe the same pattern as in the main results across all benchmarks --- models fine-tuned with TOFU consistently outperform other SFT approaches.

As a qualitative illustration of induced diversity in creative writing, we compare the outputs of Mistral-12B fine-tuned on Alpaca using Cross-Entropy versus TOFU loss, the results are gathered in Figure~\ref{fig:jokes}. When prompted to tell a funny joke, the model trained with CE tends toward redundancy, often generating very similar or nearly identical responses. In contrast, the TOFU-tuned version maintains significantly more variety between its completions. 

\begin{table*}[t]
\caption{Performance of quantized SFT models vs non-quantized across Alpaca SFT objectives on Short Stories and Small Prompts.} 
\label{tab:quantvsno}
\begin{center}
\begin{small}
\begin{sc}
\setlength{\tabcolsep}{3pt}
\begin{tabular}{clcccc}
\toprule
Bench & Method & \multicolumn{2}{c}{Llama-3.1-8B} & \multicolumn{2}{c}{Llama-3.1-8B-4bit} \\
& & D$\downarrow$ & Q$\uparrow$ & D$\downarrow$ & Q$\uparrow$ \\
\midrule
\multirow{4}{*}[-1.5ex]{SS} & CE & 23.2$_{\pm 8.2}$ & 4.8$_{\pm 0.2}$ & 22.6$_{\pm 8.5}$ & 4.7$_{\pm 0.3}$ \\[1ex]
& FL & 15.5$_{\pm 5.0}$ & 4.6$_{\pm 0.3}$ & 14.3$_{\pm 4.7}$ & 4.5$_{\pm 0.3}$ \\[1ex]
& GEM & 12.0$_{\pm 4.4}$ & 4.5$_{\pm 0.3}$ & 11.5$_{\pm 4.0}$ & 4.5$_{\pm 0.2}$ \\[1ex]
& TOFU & \textbf{11.5$_{\pm 4.3}$} & 4.5$_{\pm 0.3}$ & \textbf{11.2$_{\pm 4.2}$} & 4.5$_{\pm 0.3}$ \\
\midrule
\multirow{4}{*}[-1.5ex]{SP} & CE & 52.7$_{\pm 15.8}$ & 4.2$_{\pm 0.7}$ & 44.5$_{\pm 14.3}$ & 4.2$_{\pm 0.7}$ \\[1ex]
& FL & 32.8$_{\pm 13.1}$ & 4.2$_{\pm 0.7}$ & 28.3$_{\pm 8.8}$ & 4.1$_{\pm 0.7}$ \\[1ex]
& GEM & 33.6$_{\pm 14.7}$ & 4.1$_{\pm 0.7}$ & 25.9$_{\pm 12.3}$ & 4.1$_{\pm 0.7}$ \\[1ex]
& TOFU & \textbf{26.2$_{\pm 11.2}$} & 4.1$_{\pm 0.6}$ & \textbf{20.8$_{\pm 8.4}$} & 4.0$_{\pm 0.7}$ \\
\bottomrule
\end{tabular}
\end{sc}
\end{small}
\end{center}
\vskip -0.1in
\end{table*}

Additionally, we employ mechanistic validation, demonstrating TOFU's successful impact on simultaneously reducing ignorance and forgetting (see Figure \ref{fig:mecha}). Specifically, we ask the model to generate an integer between 1 and 5 and track the output probabilities of the top-ranked tokens. The critical comparison is between the resulting distributions of CE and TOFU, where TOFU improves the chances of correct answers without suppressing legitimate generative breadth.

\begin{table*}[t]
\caption{The performance of Alpaca SFT and UltraFeedback SFT models across different objectives on ARC (first column) and MMLU (second column). Values are measured as Accuracy scores (0-100).}
\label{tab:fact}
\begin{center}
\begin{small}
\begin{sc}
\begin{tabular}{c l c c c c c c c c c c c c}
\toprule
Bench & Method & \multicolumn{2}{c}{Mistral-12B} & \multicolumn{2}{c}{OLMo-2-13B} & \multicolumn{2}{c}{Pythia-12B} & \multicolumn{2}{c}{Llama-3.1-8B} & \multicolumn{2}{c}{Qwen-3-8B} & \multicolumn{2}{c}{Phi-4-14B} \\
 & & arc & mmlu & arc & mmlu & arc & mmlu & arc & mmlu & arc & mmlu & arc & mmlu \\
\midrule
\multirow{5}{*}[-1.5ex]{\rotatebox{90}{Alpaca}} &
CE   & 74.6 & 58.3 & 73.5 & 58.9 & 25.5 & 25.9 & 71.4 & 57.8 & 86.9 & 69.9 & 86.1 & 72.0 \\[1ex]
& FL  & 76.8 & 59.4 & 74.0 & 58.1 & 25.1 & 26.1 & 71.2 & 58.0 & 86.9 & 69.4 & 85.7 & 71.0\\[1ex]
& GEM & 76.8 & 59.4 & 74.2 & 58.6 & 23.5 & 25.2 & 72.1 & 58.3 & 87.0 & 69.7 & 86.2 & 71.9 \\[1ex]
& $\lambda$-PR  & 61.6 & 50.4 & 62.4 & 39.2 & 21.3 & 18.3 & 55.4 & 47.6 & 86.1 & 68.2 & 86.7 & 71.7 \\[1ex]
& TOFU & 73.9 & 57.8 & 74.6 & 58.7 & 24.9 & 26.0 & 72.5 & 57.3 & 86.9 & 69.2 & 85.8 & 71.9 \\
\midrule
\multirow{5}{*}[-1.5ex]{\rotatebox{90}{UF}} &
CE   & 76.3 & 59.8 & 74.4 & 59.8 & 23.6 & 21.8 & 71.8 & 58.3 & 88.1 & 70.5 & 88.7 & 74.8 \\[1ex]
& FL  & 74.3 & 58.7 & 73.7 & 59.3 & 25.2 & 22.7 & 71.1 & 58.5 & 88.3 & 70.3 & 89.0 & 75.2 \\[1ex]
& GEM & 77.3 & 59.8 & 74.1 & 59.6 & 27.1 & 21.5 & 71.8 & 58.6 & 87.9 & 70.2 & 88.9 & 74.9 \\[1ex]
& $\lambda$-PR  & 37.7 & 32.1 & 14.4 & 0.3 & 20.1 & 13.4 & 48.4 & 38.2 & 78.8 & 56.3 & 77.2 & 62.4 \\[1ex]
& TOFU & 76.1 & 58.8 & 74.1 & 59.4 & 25.3 & 24.6 & 72.3 & 58.2 & 88.2 & 70.3 & 88.7 & 74.9 \\
\bottomrule
\end{tabular}
\end{sc}
\end{small}
\end{center}
\vskip -0.1in
\end{table*}

\begin{table*}[t]
\centering
\small
\caption{SFT models across different SFT objectives on Malicious Instruct (MI) and HarmBench (HB) datasets. Safety score is measured via Attack Success Rate (0-100), where lower scores are better.}
\setlength{\tabcolsep}{3pt}
\label{tab:safety}
\begin{center}
\begin{small}
\begin{sc}
\begin{tabular}{clcccccc}
\toprule
Bench & Method & Mistral-12B & OLMo-2-13B & Pythia-12B & Llama-3.1-8B & Qwen-3-8B & Phi-4-14B \\
\midrule
\multirow{4}{*}[-1.5ex]{MI} &
CE & 86.0 & 96.8 & 98.1 & 94.7 & 54.4 & 49.3 \\[1ex]
& FL & 88.4 & 95.3 & 97.1 & 94.0 & 54.3 & 52.8 \\[1ex]
& GEM & 90.0 & 96.6 & 97.2 & 95.2 & 61.9 & 56.7 \\[1ex]
& TOFU & 87.2 & 94.9 & 97.0 & 95.3 & 56.8 & 52.4 \\
\midrule
\multirow{4}{*}[-1.5ex]{HB} &
CE & 93.2 & 96.1 & 93.8 & 95.4 & 82.2 & 79.4 \\[1ex]
& FL & 92.6 & 93.8 & 91.9 & 91.3 & 79.2 & 81.2 \\[1ex]
& GEM & 93.9 & 94.8 & 91.4 & 93.5 & 83.4 & 82.2 \\[1ex]
& TOFU & 92.6 & 91.8 & 91.3 & 93.0 & 79.5 & 81.7 \\
\bottomrule
\end{tabular}
\end{sc}
\end{small}
\end{center}
\vskip -0.1in
\end{table*}

\begin{table*}[t]
\caption{Performance of models across Alpaca SFT objectives on Short Stories and Small Prompts with extended generation length to 256 tokens. Diversity (D) is measured via Self-BLEU (0–100), where lower scores are better. Quality (Q) is measured via LLM Judge score (0–5), where higher scores are better.}
\label{tab:256}
\begin{center}
\begin{small}
\begin{sc}
\begin{tabular}{clcccccc}
\toprule
Bench & Method & \multicolumn{2}{c}{Mistral-12B} & \multicolumn{2}{c}{OLMo-2-13B} & \multicolumn{2}{c}{Llama-3.1-8B} \\

& & D$\downarrow$ & Q$\uparrow$ & D$\downarrow$ & Q$\uparrow$ & D$\downarrow$ & Q$\uparrow$ \\

\midrule

\multirow{4}{*}[-1.5ex]{SS} &
CE & 25.1$_{\pm 8.5}$ & 5.0$_{\pm 0.1}$ & 24.5$_{\pm 9.8}$ & 5.0$_{\pm 0.1}$ & 22.6$_{\pm 8.1}$ & 4.9$_{\pm 0.1}$ \\[1ex]

& FL & 15.7$_{\pm 4.9}$ & 5.0$_{\pm 0.1}$ & 16.1$_{\pm 5.7}$ & 4.9$_{\pm 0.1}$ & 14.4$_{\pm 5.3}$ & 4.9$_{\pm 0.2}$ \\[1ex]

& GEM & 14.4$_{\pm 5.6}$ & 5.0$_{\pm 0.1}$ & 13.4$_{\pm 5.0}$ & 5.0$_{\pm 0.1}$ & 11.6$_{\pm 4.0}$ & 4.9$_{\pm 0.1}$ \\[1ex]

& TOFU & \textbf{12.5$_{\pm 4.5}$} & 5.0$_{\pm 0.1}$ & \textbf{12.3$_{\pm 4.5}$} & 4.9$_{\pm 0.1}$ & \textbf{11.4$_{\pm 3.6}$} & 4.9$_{\pm 0.1}$ \\

\midrule

\multirow{4}{*}[-1.5ex]{SP} & CE & 50.9$_{\pm 15.4}$ & 4.9$_{\pm 0.3}$ & 51.6$_{\pm 15.9}$ & 4.9$_{\pm 0.2}$ & 50.3$_{\pm 15.2}$ & 4.9$_{\pm 0.3}$ \\[1ex]

& FL & 33.7$_{\pm 12.2}$ & 4.8$_{\pm 0.3}$ & 34.3$_{\pm 12.2}$ & 4.9$_{\pm 0.2}$ & 33.2$_{\pm 10.6}$ & 4.8$_{\pm 0.3}$ \\[1ex]

& GEM & 32.9$_{\pm 14.2}$ & 4.9$_{\pm 0.3}$ & 32.0$_{\pm 13.5}$ & 4.8$_{\pm 0.2}$ & 31.8$_{\pm 13.6}$ & 4.8$_{\pm 0.3}$ \\[1ex]

& TOFU & \textbf{26.8$_{\pm 11.5}$} & 4.8$_{\pm 0.3}$ & \textbf{26.1$_{\pm 10.8}$} & 4.8$_{\pm 0.3}$ & \textbf{25.6$_{\pm 10.9}$} & 4.8$_{\pm 0.3}$ \\

\bottomrule
\end{tabular}
\end{sc}
\end{small}
\end{center}
\vskip -0.1in
\end{table*}

\begin{table*}[b]
\caption{Performance of models across Alpaca SFT objectives on Short Stories and Small Prompts with extended generation length to 512 tokens. Diversity (D) is measured via Self-BLEU (0–100), where lower scores are better. Quality (Q) is measured via LLM Judge score (0–5), where higher scores are better.} 
\label{tab:512}
\begin{center}
\begin{small}
\begin{sc}
\begin{tabular}{clcccccc}
\toprule
Bench & Method & \multicolumn{2}{c}{Mistral-12B} & \multicolumn{2}{c}{OLMo-2-13B} & \multicolumn{2}{c}{Llama-3.1-8B} \\
& & D$\downarrow$ & Q$\uparrow$ & D$\downarrow$ & Q$\uparrow$ & D$\downarrow$ & Q$\uparrow$ \\
\midrule
\multirow{4}{*}[-1.5ex]{SS} & CE & 23.1$_{\pm 8.1}$ & 5.0$_{\pm 0.1}$ & 23.6$_{\pm 9.4}$ & 5.0$_{\pm 0.1}$ & 21.6$_{\pm 7.6}$ & 4.9$_{\pm 0.1}$ \\[1ex]

& FL & 15.8$_{\pm 5.0}$ & 5.0$_{\pm 0.1}$ & 15.7$_{\pm 5.7}$ & 5.0$_{\pm 0.1}$ & 14.8$_{\pm 4.7}$ & 4.9$_{\pm 0.1}$ \\[1ex]

& GEM & 14.0$_{\pm 5.4}$ & 5.0$_{\pm 0.1}$ & 13.0$_{\pm 4.6}$ & 5.0$_{\pm 0.1}$ & 11.8$_{\pm 4.2}$ & 4.9$_{\pm 0.1}$ \\[1ex]

& TOFU & \textbf{13.0$_{\pm 4.7}$} & 5.0$_{\pm 0.1}$ & \textbf{11.6$_{\pm 3.4}$} & 4.9$_{\pm 0.1}$ & \textbf{11.6$_{\pm 4.1}$} & 4.9$_{\pm 0.1}$ \\

\midrule

\multirow{4}{*}[-1.5ex]{SP} & CE & 50.9$_{\pm 15.3}$ & 4.9$_{\pm 0.3}$ & 51.4$_{\pm 16.2}$ & 4.9$_{\pm 0.2}$ & 50.5$_{\pm 15.5}$ & 4.9$_{\pm 0.3}$ \\[1ex]

& FL & 33.4$_{\pm 12.2}$ & 4.8$_{\pm 0.2}$ & 34.3$_{\pm 12.5}$ & 4.9$_{\pm 0.2}$ & 31.2$_{\pm 10.8}$ & 4.9$_{\pm 0.2}$ \\[1ex]

& GEM & 32.9$_{\pm 14.2}$ & 4.9$_{\pm 0.2}$ & 33.4$_{\pm 13.9}$ & 4.9$_{\pm 0.2}$ & 32.8$_{\pm 13.8}$ & 4.8$_{\pm 0.3}$ \\[1ex]

& TOFU & \textbf{26.6$_{\pm 11.5}$} & 4.8$_{\pm 0.3}$ & \textbf{26.3$_{\pm 10.8}$} & 4.9$_{\pm 0.2}$ & \textbf{25.8$_{\pm 10.3}$} & 4.8$_{\pm 0.3}$ \\

\bottomrule
\end{tabular}
\end{sc}
\end{small}
\end{center}
\vskip -0.1in
\end{table*}

\begin{figure}[t]
\vskip 0.2in
\begin{center}
\centerline{\includegraphics[width=\linewidth]{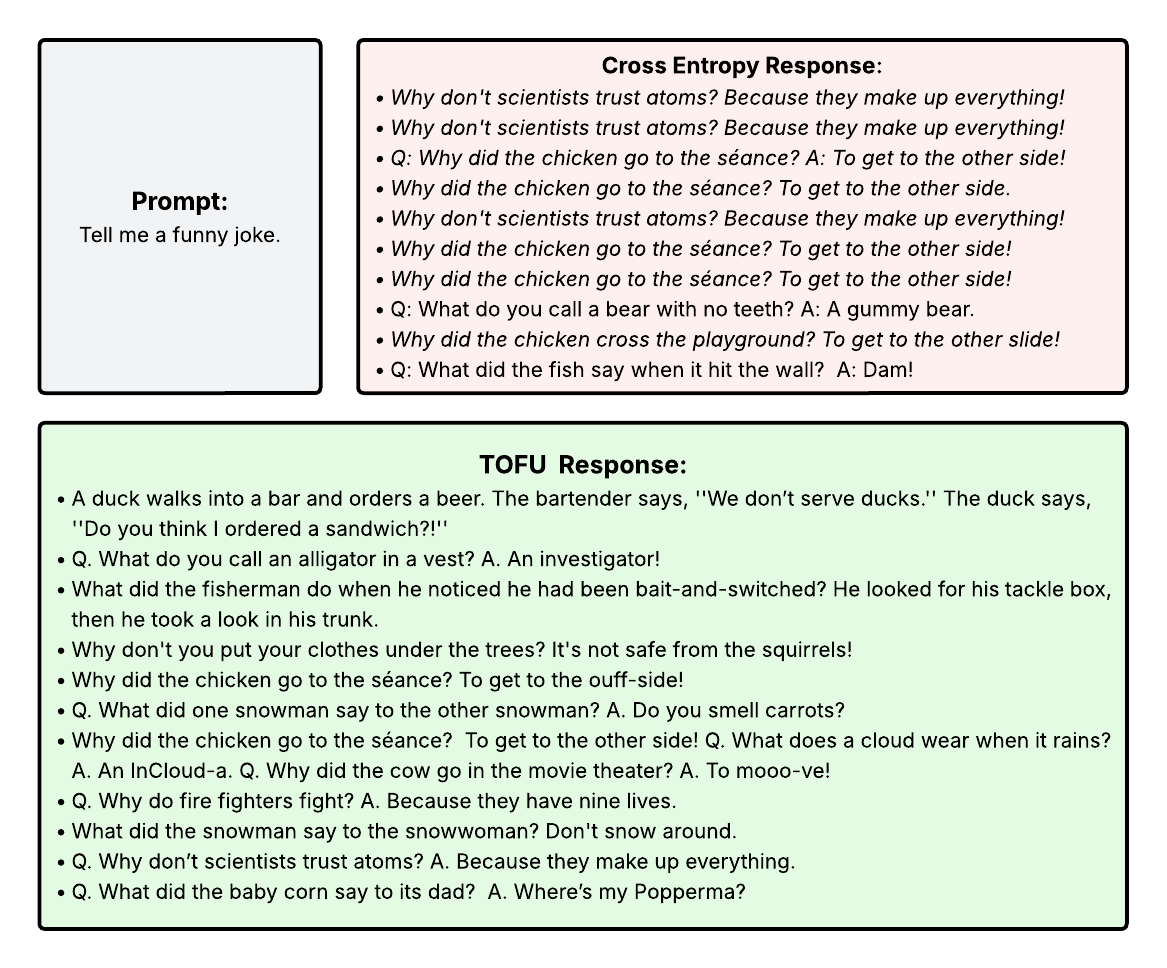}}
\caption{Example illustrating differences between different loss functions used to tune Mistral-12B.}
\label{fig:jokes}
\end{center}
\vskip -0.2in
\end{figure}

\begin{figure*}[ht]
  \vskip 0.2in
  \begin{center}
    \centerline{\includegraphics[width=\linewidth]{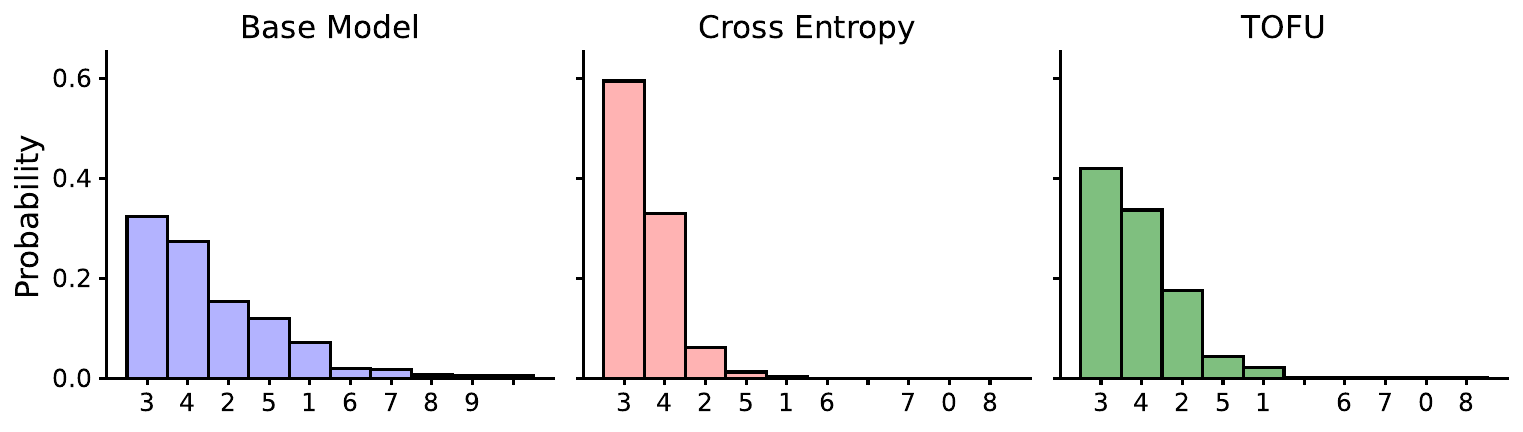}}
    \caption{
    Predicted probability distribution for the first generated token (ignoring spaces) following the prompt "Generate an integer between 1 and 5". The comparison across base model, SFT with CE, and SFT with TOFU demonstrates that TOFU prevents knowledge forgetting while simultaneously reducing the probability mass of the tail for incorrect tokens.
    }
    \label{fig:mecha}
  \end{center}
\end{figure*}

\end{document}